\documentclass[12pt]{colt2024}

\def\notes{1}

\newif\ifCOLT
\COLTtrue %

\usepackage{amsmath}
\usepackage{amssymb}

\usepackage{algorithm}
\usepackage{algpseudocode}
\usepackage{etoolbox}
\usepackage[normalem]{ulem}
\usepackage{paralist}
\setdefaultenum{(1)}{}{}{}

\usepackage{fancyhdr}
\usepackage{graphicx}
\usepackage{wrapfig}
\usepackage{ifthen}
\usepackage{bbm}
\usepackage{comment}
\usepackage{verbatim}
\usepackage{multicol}
\usepackage{mathtools} 
\usepackage[T1]{fontenc}
\usepackage[capitalize]{cleveref}

\usepackage{thmtools, thm-restate}

\usepackage{mathrsfs}
\usepackage{times}

\ifCOLT
\hypersetup{hidelinks}
\else
\usepackage{subcaption}
\usepackage{amsthm}
\usepackage[usenames,dvipsnames]{xcolor}
\usepackage[numbers,square]{natbib}
\usepackage{fullpage}

\usepackage[pdfencoding=auto, psdextra, hypertexnames=false]{hyperref}

\fi
\usepackage{autonum}
\usepackage{tikz}
\usepackage{pgfplots}
\pgfplotsset{compat=1.15}
\usetikzlibrary{arrows.meta}
\usetikzlibrary{patterns}

\newcommand{\Rjump}{\ensuremath{\exp(108 k\leverage)}}

\newcommand{\cov}{S}
\newcommand{\ols}{\beta_{\mathrm{ols}}}

\newcommand{\score}{\mathtt{SCORE}}
\newcommand{\resthresh}{\ensuremath{\mathtt{ResidualThresholding}}\xspace}

\makeatletter
\newcommand*{\algrule}[1][\algorithmicindent]{%
  \makebox[#1][l]{%
    \hspace*{.2em}%
  }
}

\newcount\ALG@printindent@tempcnta
\def\ALG@printindent{%
    \ifnum \theALG@nested>0%
    \ifx\ALG@text\ALG@x@notext%
    \else 
    \unskip
    \ALG@printindent@tempcnta=1
    \loop
    \algrule[\csname ALG@ind@\the\ALG@printindent@tempcnta\endcsname]%
    \advance \ALG@printindent@tempcnta 1
    \ifnum \ALG@printindent@tempcnta<\numexpr\theALG@nested+1\relax
    \repeat
    \fi
    \fi
}
\patchcmd{\ALG@doentity}{\noindent\hskip\ALG@tlm}{\ALG@printindent}{}{\errmessage{failed to patch}}
\patchcmd{\ALG@doentity}{\item[]\nointerlineskip}{}{}{} %
\makeatother

\ifCOLT
    \newtheorem{claim}[theorem]{Claim}
    \newtheorem{open}{Open Problem}[section]
    \newtheorem{fact}[theorem]{Fact}
    \newtheorem{observation}[theorem]{Observation}
    \newtheorem{problem}[theorem]{Problem}
    
\else
    \theoremstyle{plain}
    \newtheorem{theorem}{Theorem}[section]
    \newtheorem{claim}[theorem]{Claim}

    \newtheorem{corollary}[theorem]{Corollary}
    \newtheorem{lemma}[theorem]{Lemma}
    \newtheorem{fact}[theorem]{Fact}

    \theoremstyle{definition}
    \newtheorem{definition}{Definition}[section]

    \theoremstyle{remark}
    
\fi

\newcommand{\defeq}{\stackrel{{\mbox{\tiny def}}}{=}}
\newcommand{\eps}{\varepsilon}

\DeclareMathOperator{\tr}{\mathrm{tr}}

\newcommand{\cA}{\mathcal{A}}

\newcommand{\cM}{\mathcal{M}}
\newcommand{\cN}{\mathcal{N}}

\newcommand{\cP}{\mathcal{P}}

\newcommand{\cU}{\mathcal{U}}

\newcommand{\bbE}{\mathbb{E}}

\newcommand{\bbI}{\mathbb{I}}

\newcommand{\bbN}{\mathbb{N}}

\newcommand{\bbP}{\mathbb{P}}

\newcommand{\bbR}{\mathbb{R}}

\newcommand{\bbZ}{\mathbb{Z}}

\ifnum\notes=1
\newcommand{\mynote}[3]{\marginpar{\tiny \sf \color{#1} {#2}: {#3}}}

\else 
\newcommand{\mynote}[3]{}
\fi

\newcommand{\paren}[1]{{\left( {#1} \right)}}

\newcommand{\braces}[1]{{\left\{ {#1} \right\} }}

\newcommand{\ip}[2]{{\left\langle {#1}, {#2} \right\rangle}}
\newcommand{\indicator}[1]{\mathbbm{1}\{ {#1} \} }

\newcommand{\supp}[1]{\ensuremath{\mathrm{supp}\paren{{#1}}}}

\newcommand{\PTR}{\ensuremath{\cM_{\mathrm{PTR}}^{\eps,\delta}}}
\newcommand{\PTRepsdelta}[2]{\ensuremath{\cM_{\mathrm{PTR}}^{#1, #2}}}
\newcommand{\pass}{\ensuremath{\mathtt{PASS}}}
\newcommand{\fail}{\ensuremath{\mathtt{FAIL}}}

\newcommand{\R}{R}

\newcommand{\yy}{y}

\newcommand{\stablecovariance}{\ensuremath{\mathtt{StableLeverageFiltering}}\xspace}
\newcommand{\algname}{\ensuremath{\mathtt{ISSP}}\xspace}
\newcommand{\weightedOLS}{\ensuremath{\mathtt{WeightedOLS}}}
\newcommand{\stableOLS}{\ensuremath{\mathtt{StableResidualFiltering}}\xspace}
\newcommand{\privateOLS}{\algname}

\newcommand{\leverage}{L}
\newcommand{\residual}{R}

\newcommand{\RR}{\mathbb{R}}

\newcommand{\blank}{\mathrel{\,\cdot\,}}

\DeclareMathOperator{\diag}{diag}

\DeclareMathOperator*{\argmax}{\arg\!\max}
\DeclareMathOperator*{\argmin}{\arg\!\min}
\DeclarePairedDelimiter\angles{\langle}{\rangle}
\DeclarePairedDelimiter\ceil{\lceil}{\rceil}

\DeclarePairedDelimiter\del{\lparen}{\rparen}
\let\set\relax
\DeclarePairedDelimiter\set{\lbrace}{\rbrace}
\DeclarePairedDelimiter\abs{\lvert}{\rvert}
\usepackage{etoolbox}
\DeclarePairedDelimiterX\norm[1]\lVert\rVert{\ifblank{#1}{\blank}{#1}}
\DeclarePairedDelimiter\sbr{\lbrack}{\rbrack}

\DeclarePairedDelimiter\intoc{\lparen}{\rbrack}

\newcommand{\textmultinom}[2]{\bigl(\kern-0.3em{\binom{#1}{#2}}\kern-0.3em\bigr)}
\newcommand{\displaymultinom}[2]{\left(\kern-0.5em{\binom{#1}{#2}}\kern-0.5em\right)}

\providecommand\given{\errmessage{\noexpand\given used outside \noexpand\DelGivenX}}
\DeclarePairedDelimiterX{\DelGivenX}[1]{\lparen}{\rparen}{%
  \renewcommand\given{\mathclose{}\,\delimsize\vert\allowbreak\,\mathopen{}}#1%
}
\DeclarePairedDelimiterX{\SbrGivenX}[1]{\lbrack}{\rbrack}{%
  \renewcommand\given{\mathclose{}\,\delimsize\vert\allowbreak\,\mathopen{}}#1%
}
\DeclarePairedDelimiterX{\SetGivenX}[1]{\lbrace}{\rbrace}{%
  \renewcommand\given{\mathclose{}\,\delimsize\vert\allowbreak\,\mathopen{}}#1%
}
\providecommand\from{\errmessage{\noexpand\from used outside \noexpand\DelGivenX}}
\DeclarePairedDelimiterX{\DelFromX}[1]{\lparen}{\rparen}{%
  \renewcommand\from{\mathclose{}\,\delimsize\Vert\allowbreak\,\mathopen{}}#1%
}

\makeatletter
\newcommand{\spx}[1]{%
\if\relax\detokenize{#1}\relax
\expandafter\@gobble
\else
\expandafter\@firstofone
\fi
{^{#1}}%
}

\makeatother

\newcommand{\E}{\mathbb{E}\SbrGivenX}

\newcommand{\vsmask}[2]{\Pi_{#2}\del{#1}}
\newcommand{\setcompl}[1]{\overline{#1}}

\newcommand{\dpsd}{d_{\mathrm{PD}}}
\newcommand{\bigO}{O}

\newcommand{\T}{\top}

\title[Insufficient Statistics Perturbation]{Insufficient Statistics Perturbation:\\ Stable Estimators for Private Least Squares}

\coltauthor{\Name{Gavin Brown\nametag{\thanks{Paul G. Allen School of Computer Science and Engineering, University of Washington. Part of this work was done while G.B. was at Boston University.}}}
\Email{grbrown@cs.washington.edu}\\
\Name{Jonathan Hayase\nametag{\footnotemark[1]}}
\Email{jhayase@cs.washington.edu}\\
\Name{Samuel Hopkins\nametag{\thanks{Department of Electrical Engineering and Computer Science, Massachusetts Institute of Technology.}}}
\Email{samhop@mit.edu}\\
\Name{Weihao Kong\nametag{\thanks{Google Research}}}
\Email{kweihao@gmail.com}\\
\Name{Xiyang Liu\nametag{\footnotemark[1]}}
\Email{xiyangl@cs.washington.edu}\\
\Name{Sewoong Oh\nametag{\footnotemark[1]}}
\Email{sewoong@cs.washington.edu}\\
\Name{Juan C. Perdomo\nametag{\thanks{Harvard University}}}
\Email{jcperdomo@g.harvard.edu}\\
\Name{Adam Smith\nametag{\thanks{Department of Computer Science, Boston University.}}}
\Email{ads22@bu.edu}
}

\begin{document}

\maketitle

\begin{abstract}%
    We present a sample- and time-efficient differentially private algorithm for ordinary least squares, with error that depends linearly on the dimension and is independent of the condition number of $X^\top X$, where $X$ is the design matrix.
    All prior private algorithms for this task require either $d^{3/2}$ examples, error growing polynomially with the condition number, or exponential time. 
    Our near-optimal accuracy guarantee holds for any dataset with bounded statistical leverage and bounded residuals.
    Technically, we build on the approach of \cite{brown2023fast} for private mean estimation, adding scaled noise to a carefully designed stable nonprivate estimator of the empirical regression vector.
\end{abstract}

\section{Introduction}

We present a sample- and time-efficient differentially private algorithm for ordinary least squares (OLS) regression.
Central throughout the theory and practice of data science, OLS is used in numerous domains, ranging from causal inference, to control theory, to (of course) supervised learning. 

Given covariates $X\in \mathbb{R}^{n\times d}$ and responses $y\in \mathbb{R}^n$, the OLS estimator is defined as
\begin{align}
\beta_{\mathrm{ols}} \;=\; \del*{X^\T X}^{-1} X^\T y\;.
    \label{eq:def_OLS}
\end{align}
Among the many reasons for the popularity of OLS is the fact that it is a statistically and computationally efficient way of solving linear regression. Speaking informally, OLS has low excess error whenever the number of samples $n$ is as large as the problem dimension $d$. Crucially, its statistical performance does not depend on the condition number $\kappa(X^\T X)$, the ratio between the maximum and minimum eigenvalues. %
Furthermore, it can be computed in closed-form using only basic linear-algebraic operations, with no need for the subtle hyperparameter tuning often inherent in first-order methods.

Given its widespread use in the analysis of personal data, there is a long line of work giving differentially private algorithms to approximate OLS.
However, designing practical and efficient algorithms for this problem has been a particularly challenging endeavor; so far there are no clear answers even in the $d=1$ case when $x$ and $y$ are both scalars \citep{alabi2020differentially}. Existing algorithms for DP regression suffer from one of three limitations: they either have poor dimension dependence in their sample complexity, place unnaturally restrictive assumptions on the geometry of the data, or run in exponential time. 

In terms of private algorithm design, one natural and well-established approach is \emph{sufficient statistics perturbation}, which privately produces separate estimates of $X^\T X$ and $X^\T y$ and then combines them to produce a single parameter estimate.
Such approaches are often efficient and some versions come with formal accuracy guarantees. 
An exemplar in this line is the \emph{AdaSSP} algorithm of \cite{wang2018revisiting}.
The central drawback in all these algorithms, however, is the sample complexity as $d$ grows: privately producing an accurate estimate of $X^\T X$ requires roughly $d^{3/2}$ samples~\citep{dwork2014analyze}.
Furthermore, many approaches within this class add noise proportional to the worst-case sensitivity of $X^\T X$ and $X^\T y$ \citep[see, e.g.,][]{sheffet2017differentially}. 
To deal with the fact that this sensitivity is \emph{unbounded} in the case of real-valued data, these results assume uniform norm bounds on the covariates $x$ and responses $y$ (e.g., $\norm{x}\leq B_y, |y| \leq B_y$). While conceptually simple, they fail to capture the intrinsic complexity of the problem and do not satisfy natural properties like scale invariance.%

Another approach comes via private optimization, searching for a parameter estimate that approximately minimizes the sum of squared errors.
Despite a wealth  private convex optimization methods that can be applied directly to linear regression, 
off-the-shelf approaches again require $d^{3/2}$ samples for accurate estimates.
A notable exception is the recent work from~\cite{varshney2022}, whose algorithm based on private gradient descent succeeds with only roughly $d$ samples.
However, its error grows with the square of the condition number, a high price to pay for many problems.
A polynomial dependence on $\kappa(X^\T X)$ is inherent in private first-order optimization for linear regression, as the smoothness of the optimization task is directly linked to the condition number.

The only known approach that avoids these two issues is the exponential-time algorithm of \cite{liu2022differential}, which comes from the framework they call \emph{high-dimensional propose-test-release} (HPTR), after the \emph{propose-test-release} (PTR) framework of \cite{DworkL09}.

We see the mirror of this story in private mean estimation, where \cite{duchi2023fast} and \cite{brown2023fast} recently gave the first sample- and time-efficient private algorithms with error guarantees that adapt to the covariance of the data.
All prior private algorithms achieving this guarantee require $d^{3/2}$ examples, error growing polynomially with the condition number of the covariance, or exponential time.

In this work, we build on the work of \cite{brown2023fast} and present the first computationally efficient (in fact, practically implementable) differentially private estimator for linear regression with sample complexity independent of $\kappa(X^\T X)$ and the optimal linear dependence on the dimension $d$.
Furthermore, we make no use of norm bounds. We establish its  utility  under the ``textbook'' conditions one would typically require to run OLS in the non-private setting. More specifically, the algorithm is accurate as long no observation has high statistical leverage or a large residual, formalized in \cref{def:goodness_set}.

\subsection{Our Results}

In this work, we introduce a new algorithm, \algname, for differentially private linear regression. At a high level, \algname works in two main phases. In the first, we search for a reweighting of the dataset such that running OLS on this reweighted version is roughly stable. Having successfully found this set of weights, we simply compute the OLS solution on this weighted version of the data and add appropriately  shaped Gaussian noise to the solution. While the approach is conceptually very simple, establishing its correctness requires several significant technical advances.

Our estimator satisfies \emph{differential privacy} (DP), the gold standard for privacy protection in statistical data analysis. DP requires that an algorithm provides approximately the same output on any datasets that differ in only one entry.%
\begin{definition}[\cite{DworkMNS06}]
Let $\mathcal{X}$ and $\mathcal{Y}$ be sets. An algorithm $\mathcal{A}: \mathcal{X}^n \rightarrow \mathcal{Y}$
is $(\eps, \delta)$-differentially private if for every $x = (x_1, \dots, x_n) \in \mathcal{X}^n$ and $x' = (x_1',\dots, x_n') \in \mathcal{X}^n$ such that $x, x'$ agree on all but one coordinate and for all $Y\subseteq \mathcal{Y}$, 
\begin{align}
    \bbP\sbr[\big]{\cA(x) \in Y} &\;\;\le\;\; e^\eps \,  \bbP\sbr[\big]{\cA(x') \in Y} + \delta\;. 
\end{align}
\end{definition}

 One of the core advances we make, in light of most previous results on DP regression, is that we do not require any norm bounds on the data. We only assume the types of conditions a circumspect statistician would always verify to ensure that OLS is a sensible procedure.
In particular, we establish the utility of our estimator whenever the dataset is free of outliers, or ``good.'' 

\begin{definition}[\(\del{\leverage, \residual}\)-goodness]\label[definition]{def:goodness_set}
    Fix parameters \(\leverage, \residual > 0\).
    A dataset \(\del{X, y} \in \bbR^{n \times d}\times \bbR^n\) is called \emph{\(\del{\leverage, \residual}\)-good} if \(X^\T X\) is invertible and the following conditions hold for all \(i \in \sbr{n}\).
    \begin{enumerate}
        \item Bounded leverage: $x_i^\T (X^\T  X)^{-1} x_i \le \leverage$.
        \item Bounded residuals: $\abs[\big]{\ip{x_i}{\ols} - y_i} \le \residual$.%
    \end{enumerate}
\end{definition}

\begin{figure}
    \centering
    \subfigure[Large residual][b]{
        \begin{tikzpicture}[x=0.4cm,y=0.4cm]
        \draw [line width=1pt,<->] (-4,-1.1428571428571428)-- (4,2.857142857142857);
        \fill (-3,-1.5) circle[radius=2pt];
        \fill (-2,-1.0) circle[radius=2pt];
        \fill (-1,-0.5) circle[radius=2pt];
        \fill (0,7) circle[radius=2pt];
        \fill (1,0.5) circle[radius=2pt];
        \fill (2,1.0) circle[radius=2pt];
        \fill (3,1.5) circle[radius=2pt];
        \end{tikzpicture}
        \label{fig:residual_outlier}
    } \hspace{8em}
    \subfigure[High leverage][b]{
        \begin{tikzpicture}[x=0.4cm,y=0.4cm]
        \draw [line width=1pt,<->] (-2,-0.8571428571428)-- (8,2);
        \fill (-1.5,-3) circle[radius=2pt];
        \fill (-1.0,-2) circle[radius=2pt];
        \fill (-0.5,-1) circle[radius=2pt];
        \fill (7,0) circle[radius=2pt];
        \fill (0.5,1) circle[radius=2pt];
        \fill (1.0,2) circle[radius=2pt];
        \fill (1.5,3) circle[radius=2pt];
        \end{tikzpicture}
        \label{fig:leverage_outlier}
    }
    \caption{
        As famously illustrated by \cite{anscombe1973graphs}, a point may be influential because of its large residual (\emph{a}) or its large leverage (\emph{b}).
        \cref{def:goodness} controls both quantities.
    }
\end{figure}
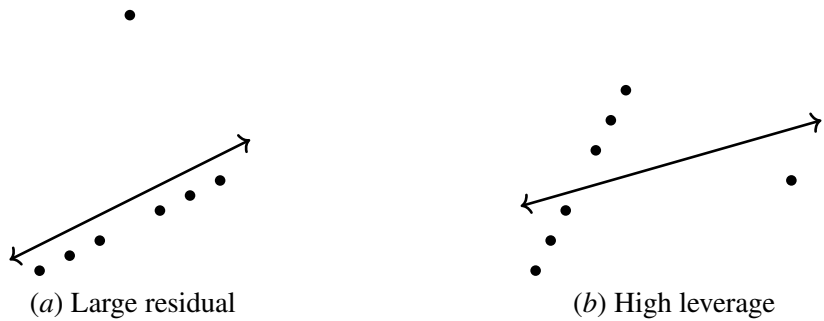

Note that both of these conditions hold in various natural, well-studied settings. For instance, when $x,y$ are both subgaussian and drawn from a well-specified linear model with true parameter $\beta^*$, these conditions hold with high probability when $L \approx d/n$ and $R \approx \sigma$ where $\sigma^2 = \E{(y - \langle \beta^*, x\rangle)^2}$ (see \cref{thm:main}). 
This idea, of outliers being observations with high leverage or large residuals,
is quite classical and found across standard texts.
For instance, a standard rule-of-thumb identifies high-leverage points as those with leverage greater than $2d/n$ or $3d/n$ \citep{hoaglin1978hat,velleman1981efficient,mendenhall2003second}.
The precise forms of stability we need to ensure privacy, however, are far from classical.
They require a carefully designed algorithm, which we elaborate in \cref{sec:tech}. 
Such a stable estimator, in turn, implies that we can achieve differential privacy with small amounts of noise.

In these well-behaved instances, 
our mechanism takes a most straightforward form:
it returns the OLS solution plus a small amount of Gaussian noise. 
\begin{claim}
\label[claim]{obs:output_on_good_data} 
    If $(X,y)$ is $(\leverage,\residual)$-good for parameters   $R > 0 $ and \(\leverage \leq  
    c'\varepsilon^2 \log^{-2} (1/\delta) \), for some constant $c'$, then %
    \(\algname(X,y;\eps,\delta,\leverage,\residual)\) 
    releases a sample drawn from \(\cN\del[\big]{\ols, c^2 (X^\T X)^{-1}}\), where $ c^2=\Theta\left(\leverage \,\residual^2\, {\log (1/\delta)}/{\eps^2}\right).$
\end{claim}
Attentive readers will detect a modest sleight-of-hand: $(\leverage,\residual)$-goodness is a property of the data and a priori unknown, yet the algorithm gets $\leverage$ and $\residual$ as inputs!
Nevertheless, an analyst with beliefs about the data generation process can set these parameters appropriately.
The maximum leverage score does not depend on the scale of the data, only its concentration properties.
Since it lies within $[0,1]$, one might pick the $\leverage$ hyperparameter adaptively by calling \algname a handful of times.
Similarly, if the analyst believes the labels are generated by a process such as $y_i \gets \ip{x_i}{\beta^*} + \cN(0,\sigma^2)$, they can privately produce an accurate estimate $\sigma$ using standard tools (see \cref{app:sigma}).
We believe alternative standardized or studentized definitions could remove this need for prior knowledge about $\sigma$.
These alternatives would likely increase the complexity of our proofs.

The difficulty in our work lies in proving that \algname is differentially private (\cref{thm:main}); reasoning about utility is simple once we have \cref{obs:output_on_good_data}.
More specifically, seeing how the output distribution on good data matches standard statistical practice (and classical CLT-like analyses of OLS), we can quickly derive error bounds. 
For instance, in the simplest case of \emph{fixed design}, where we consider only the randomness of the labels generated from a well-specified linear model $y_i = \ip{x_i}{\beta^*} + \cN(0,\sigma^2)$, we have $\ols \sim \cN(\beta^*, \sigma^2\cdot (X^\T X)^{-1})$. Hence, from \cref{obs:output_on_good_data}, we see that, relative to the empirical OLS solution, the private estimator is just a slightly noisier version of the true parameter (and has the same kind of error covariance). 

More formally, we can analyze the mean squared error (MSE) of our algorithm on any good dataset.
\begin{corollary}
\label[corollary]{cor:fixed_design_closeness}
Fix $(X,y)$, $\eps>0$, and $\delta\in[0,1]$.
If $(X,y)$ is $(\leverage,\residual)$-good for \(\leverage \leq  c'{\varepsilon^2 \log^{-2} (1/\delta)}\) for some constant $c'$ and \(R>0\), then \(\algname\del{X,y;\eps,\delta,\leverage,\residual}\), releases $\tilde\beta$ such that, for some absolute constant $c'$,
    \begin{align}
        \bbE \sbr*{\frac{1}{n} \norm[\big]{y - X\tilde{\beta} }^2}\;\; =\;\; \frac{1}{n} \norm[\big]{y - X\ols}^2 + c' \leverage \residual^2\, \frac d n\, \frac{  \log 1/\delta}{\eps^2}.
    \end{align}
\end{corollary}
\begin{proof}
By \cref{obs:output_on_good_data}, we have $\ols - \tilde\beta = c\cdot (X^\T X)^{-1/2} u$ for $u\sim \cN(0,\bbI)$. 
We then expand:
\begin{align}
        \bbE_u \norm[\big]{y - X\tilde\beta}^2 &= \bbE_u \norm[\big]{y - X\ols + X(\ols-\tilde\beta)}^2 \\
        &= \bbE_u \norm[\big]{y - X\ols + X \cdot c (X^\T X)^{-1/2} u}^2 \\
        &=\norm{y - X\ols}^2 + c^2 \cdot \bbE_u \left[u^\T (X^\T X)^{-1/2} X^\T X (X^\T X)^{-1/2} u\right],
\end{align}
where the cross terms drop out as $u$ is independent and mean-zero. 
The matrices cancel and we are left with $\bbE [u^\T u]$, which is exactly $d$.
\end{proof}

We emphasize that this result holds without any assumption that the data arises from a specific family of distributions. 
It assumes $(X,y)$ is fixed and $(L,R)$-good to bound the difference from the empirical OLS solution on $(X,Y)$. 
However, if we do add such distributional assumptions, it is easy to show that our algorithm produces a private parameter estimate that closely approximates the true regression parameter. We state this fact as part of the following theorem, our main result.

\begin{theorem}[Main Theorem]\label{thm:main}
    Fix \(\varepsilon, \eta \in \del{0, 1}\), \(\delta \in \intoc{0, \varepsilon/10}\), and \(n, d \in \bbN\).
    \algname takes a dataset \(\del{X, y} \in \RR^{n \times d} \times \RR^n\), privacy parameters \(\varepsilon, \delta\), and outlier thresholds \(\leverage_0, \residual_0\).
    \begin{enumerate}
        \item %
        \algname is \(\del{\varepsilon, \delta}\)-differentially private.
        \item Let \(X \in \RR^{n \times d}\) be drawn i.i.d. from a \(d\)-dimensional subgaussian distribution \(\mathcal{D}\) with mean \(0\) 
        and covariance \(\Sigma \succ 0\).
        Let \(y_i = \beta^\T x_i + z_i\) where the \(z_i\) are drawn i.i.d. from a subgaussian distribution with mean 0 and variance \(\sigma^2\)
        (see \cref{def:subgaussian_vector} in \cref{sec:preliminaries}).
        \footnote{We state the utility guarantees of our estimator for the case where data is drawn from a well-specified linear model to simplify the presentation and enable direct comparisons to previous work. However, as per \cref{cor:fixed_design_closeness}, on good data our algorithm is always close to the OLS solution. Hence, we can prove closeness to the population quantity whenever the OLS solution concentrates.}
        If 
        \(\leverage_0 = \widetilde{\Theta}\del*{d/n}\), \(\residual_0=\widetilde{\Theta}\del{\sigma}\),
        and
        \[n \;\;=\;\; \widetilde{\Omega}\del*{\frac{d}{\alpha^2} + \frac{d\sqrt{\log1/\delta}}{\alpha\varepsilon} + \frac{d\del{\log 1/\delta}^2}{\varepsilon^2}}\;, \]
        with a large enough constant for some $\alpha>0$, then  \algname %
         returns \(\tilde{\beta}\) such that, with high probability, 
         $$\norm[\big]{\tilde\beta - \beta}_\Sigma \;\; \leq \;\; \sigma \alpha.$$
        Here, \(\widetilde{\Theta}\) and \(\widetilde{\Omega}\) hide logarithmic factors in \(1/\alpha\), \(\log\del{1/\varepsilon}\), and \(\log\del{1/\delta}\) as well as polynomial factors of the subgaussian parameters.
        \item \cref{alg:private_OLS} can be implemented to require
             one product of the form $A^\top A$ for $A\in \R^{n\times d}$,
             one product of the form $A B$ for $A\in R^{n\times d}$ and $B\in \R^{d\times d}$,
             one inversion of a positive definite matrix in $\R^{d\times d}$; and
             further computational overhead of $\tilde{O}(nd/\eps)$.
    \end{enumerate}
\end{theorem}
Informally, this running time corresponds to $\widetilde{O}(nd^{\omega-1} + nd/\eps)$, where $\omega<2.38$ is the matrix multiplication exponent.
For modest privacy parameters, the running time of our algorithm is dominated by the time needed to compute the nonprivate OLS solution itself.

This is the first computationally efficient algorithm whose sample complexity is linear in $d$ and has no dependence on the condition number $\kappa(X^\top X)$. 
This almost matches the best known sample complexity of an exponential-time algorithm from \cite{liu2022differential}; we have an additional $d(\log(1/\delta))^2/\eps^2$ term, but this term does not depend on the final accuracy $\alpha$. %

\begin{algorithm2e}[h]
\SetAlgoLined
\SetKwInOut{Input}{Input}
\SetKwInOut{Require}{require}
\Input{dataset \(\del{X, y}\); privacy parameters \(\varepsilon, \delta\), outlier thresholds \(\del{\leverage_0, \residual_0}\)}
\BlankLine
\(k\gets \ceil*{\frac{12 \log 3/\delta}{\varepsilon}} + 8; \hspace{0.5cm} c^2 \gets 56448\exp\del[\big]{432 k^2 \leverage_0} \leverage_0\residual_0^2\cdot \frac{\log \del{12/\delta}}{\varepsilon^2}\)\;
\If{\(\leverage_0 > 1/\del{96k}\) or \(\leverage_0 > 3\varepsilon/\del{56\log 12/\delta}\)}{
\KwRet \(\fail\)\;
}
\(\score_1, w \gets \stablecovariance\del{X, \leverage_0, k}\)\tcp*{\cref{alg:stablecovariance}}
\(\score_2, v \gets \stableOLS\del{X, y, w, \leverage_0, \residual_0, k}\)\tcp*{\cref{alg:stable_OLS}}
\eIf{\(\PTRepsdelta{\varepsilon/3}{\delta/3}\del{\max\set{\score_1, \score_2}} = \fail\)}{
    \KwRet \(\fail\)\;
}{
\(\cov_v \gets X^\T \diag\del{v} X\)\;
\(\hat{\beta} \gets (\cov_v)^{-1} X^\T \diag\del{v} y\)\tcp*{OLS weighted by $v$} %
\KwRet \(\tilde{\beta} \sim \cN\del[\big]{\hat{\beta}, c^2\cov_v^{-1}}\)\;
}
\caption{InSufficient Statistics Perturbation (\privateOLS)}\label{alg:private_OLS}
\end{algorithm2e}

We now briefly sketch the steps of the proof and discuss the paper's organization.
We establish \cref{thm:main}'s subclaims in \cref{lemma:main_privacy,lemma:accuracy_main,lemma:main_time}.
As outlined above, the utility analysis is straightforward once we have \cref{obs:output_on_good_data} in hand.
The full analysis is presented in \cref{sec:utility}.
It is easy to see that \cref{alg:private_OLS} runs in polynomial time. 
In \cref{sec:running_time}, we analyze a careful implementation.

The bulk of the work comes in the privacy analysis. 
In \cref{sec:greedy_analysis}, we analyze the greedy residual thresholding algorithm, with the main result about that algorithm being \cref{claim:intertwining}, the ``intertwining'' property.
Then, in \cref{sec:stable_analysis}, we establish our guarantees for \stableOLS.
The main results about \stableOLS are \cref{claim:score_low_sensitivity}, which says that the score is low-sensitivity, and \cref{claim:weights_stable}, which says that the weights are stable.
\cref{sec:main_analysis} pulls these together to establish the privacy of \algname.

\cref{sec:related} covers additional related work.
\cref{sec:preliminaries} provides necessary preliminaries.
\cref{sec:deferred_proofs} contains proofs deferred from the main text.
\cref{app:sigma}, via standard tools, shows how to privately estimatie $\residual$.
\cref{sec:lower_bound} contains details on the lower bound of \citet{cai2023score}.

\subsection{Optimality}

For modest values of the privacy parameters, the error of our algorithm is dominated by the empirical error of OLS. 
Informally speaking, we obtain privacy ``for free.''

Formally, our error guarantees are close to tight for random-design regression with subgaussian covariates and subgaussian label noise.
Suppressing constants and logarithmic factors other than $\log 1/\delta$, \cref{thm:main} says that we can achieve $\norm{\tilde \beta - \beta}_\Sigma \le \sigma \alpha$ with high probability with
\begin{align}
    n \approx \frac{d}{\alpha^2} + \frac{d\sqrt{\log 1/\delta}}{\alpha \eps} + \frac{d(\log 1/\delta)^2}{\eps^2}.
\end{align}
Known lower bounds imply this task requires
\begin{align}
    n \gtrsim \frac{d}{\alpha^2} + \frac{d}{\alpha \eps} + \frac{\log 1/\delta}{\eps}. \label{eq:exact_lbd}
\end{align}
The first term corresponds to the classical analysis of OLS.
The second term was established by \citet{cai2023score} and holds even for parameter estimation in $\ell_2$ norm; see \cref{sec:lower_bound} for a more detailed discussion.
The third term, the minimal number of samples required to produce any estimate, is from \citet{karwa2017finite} and holds even for one-dimensional mean estimation with known variance.
The exponential-time algorithm of \citet{liu2022differential} nearly matches all three terms.
For constant $\eps$ and $\delta = 1/\mathrm{poly}(n)$, our algorithm's error guarantee in this setting is tight up to logarithmic factors.
An exciting topic for future work is determining the existence or impossibility of efficient algorithms with error matching \cref{eq:exact_lbd} up to constant factors.

\subsection{Techniques}
\label{sec:tech}

At a high level, our algorithm follows the blueprint for private mean estimation laid out by \cite{brown2021covariance} and made computationally efficient by \cite{duchi2023fast} and \cite{brown2023fast}.
Our approach closely follows that of \cite{brown2023fast}, henceforth BHS.
We now sketch our algorithm, discuss how our analysis differs from that of BHS, and investigate how the notions we use are, in a sense, ``correct'' for the task of private least squares.

\paragraph{Overview of \algname} Perhaps the most natural approach for private estimation of regression coefficients is to perturb the ordinary least square estimator, $\ols$. However, without restrictions on the data, the sensitivity of $\ols$ is unbounded. 
Our key observation is that, on datasets with bounded leverage and bounded residuals,
the OLS solution is actually quite stable. 
If we could restrict our inputs to only such outlier-free data sets, we might hope to release $\ols$ plus noise with shape $(X^\top X)^{-1}$.

While this would provide accuracy, it fails on privacy: we must accommodate worst-case data. 
We use the PTR framework of \cite{DworkL09} to test if our input contains a large good subset. 
We propose a greedy pruning algorithm which, in each iteration, removes the data point with the largest residual and recomputes OLS  on the remaining data. 
Similar approaches abound in the literature on robust statistics, but we prove key new properties about how this algorithm behaves across adjacent data sets and different outlier thresholds.

\paragraph{Adaptively selecting outlier thresholds}
Our algorithm takes as input target bounds $\leverage$ and $\residual$ for the leverage and residuals, respectively.
This simplifies our analysis but is not strictly necessary.
The maximum leverage can only lie within the interval $[0,1]$, so one could imagine calling \algname repeatedly within this space (via a well-chosen grid or binary search) to find an appropriate setting, perhaps via a small validation set.
Independently, one could privately learn an appropriate value for $\residual$ directly through standard techniques; we give a complete description in \cref{app:sigma}.

\paragraph{Proof techniques} 
While our work builds on a long line of research connecting robust statistics and differential privacy, it especially relies on the 
recent algorithmic approach of BHS, who gave improved algorithms for private mean and covariance estimation.
At a bird's eye view, our recipe for private linear regression follows the main ideas behind the mean estimator of BHS. However, key parts of the implementation and analysis differ significantly in the more complicated linear regression setting. %

We start by discussing the ways in which our main proof strategy is similar to BHS.
As mentioned previously, we introduce a notion of ``good'' outlier-free datasets for linear regression.
We repeatedly call a greedy algorithm to find a series of good weight vectors 
across a range of carefully chosen outlier thresholds. 
We use these vectors to privately test that our input data is sufficiently close to the good set and to finally produce a vector of weights  over the input.
Crucially, this weight-finding procedure is stable: if run on any adjacent dataset, it would produce a vector that is close in $\ell_1$ distance.

When adapting their analysis, however, we run into immediate issues. 
For both their definitions of good set, BHS prove a number of strong properties that are false in the context of regression.
For mean and covariance estimation, the good sets are unique (i.e., for any dataset and outlier threshold, there exists a unique largest good set), are directly found by the natural greedy algorithm, and enjoy a form of monotonicity (e.g., introducing a new point to the dataset cannot alter the good set very much). 
For the definition we use, there is no unique ``largest good set'' which introduces significantly more complexity in the analysis.
What's more, adding a single point may significantly affect the downstream choices made by the greedy algorithm which further complicate the relevant stability calculations. 

In more detail, a key step in the BHS analysis establishes the following ``intertwining'' property in the context of mean and covariance estimation. 
Suppose we call the greedy algorithm on a dataset $D$ with outlier threshold $B$ and find a largest good subset $S\subseteq [n]$.
If we then call the same algorithm on an adjacent dataset $D'$ with a slightly larger outlier threshold $B'$, the largest good subset $T$ will satisfy the property that $S\subseteq T$ (ignoring the index that differs between $D$ and $D'$).

We establish an analogous statement (\cref{claim:intertwining}) about the output of our greedy residual thresholding, \cref{alg:residual_thresholding}, even though we cannot prove the same uniqueness and monotonicity statements.
More specifically, we develop a novel regression-specific argument  that uses the closed-form expressions describing how the least squares solution changes when an observation is added or removed.
The exact arguments are formalized in \cref{claim:final:goodness_after_removal} and \cref{claim:add_weight_outside_support}, but we  sketch the ideas here.

For a dataset $(X,y)$ and index $i\in [n]$, let $\hat y_i = x_i^\T \ols$ be the fitted value.
We denote by $e_i$ the $i$-th residual: $e_i \defeq \hat y_i - y_i$.
Recall the \emph{hat matrix}:
\begin{align}
    H \defeq X (X^\T X)^{-1} X^\T , 
        \label{eq:hat_matrix}
\end{align}
so called because it maps the true labels to their ``hat'' values: $\hat y = H y$.
The \emph{leverage scores} (also known as \emph{sensitivities} or \emph{self-influences}) form its diagonal entries, while its off-diagonal entries will be called (by us) the \emph{cross-leverage scores}:
\begin{align}
    h_i = H_{i,i} &= x_i^\T (X^\T X)^{-1} x_i  \quad\text{and}\quad  
    H_{i,j} = x_i^\T (X^\T X)^{-1} x_j.
\end{align}
Note that, by Cauchy--Schwarz, the cross-leverages are no larger in magnitude than the leverages.

What happens if we remove an observation, say, $(x_j,y_j)$, from the dataset?
This takes the form of a rank-one update. Applying the Sherman-Morrison formula we can derive closed-form expressions for the changes in the OLS solution as well as (for any $i\in[n]$) the leverage score and residual of point $i$ after removing $j$.
Using the subscript ``$(-j)$'' to denote the quantity after removal, we have
\begin{align}
     h_i - h_{i(-j)} &=  - \frac{H_{i,j}^2}{1 - h_j} \\
    \ols  - \beta_{\mathrm{ols}(-j)} &= \frac{(X^\T X)^{-1} x_j}{1-h_j}\cdot \del[\big]{\ip{x_j}{\ols} - y_j}\\
    e_i - e_{i(-j)}  &= H_{i,j}\cdot \frac{e_j}{1-h_j}.
\end{align}
These well-known formulas have elementary derivations; the second and third correspond to the DFBETA (``difference in $\beta$'') and DFFITS (``difference in fits'') regression diagnostics \cite[see textbooks such as][]{mendenhall2003second,belsley2005regression,huber2011robust}.
All three seamlessly generalize to the case where the points are weighted. 
We can reuse them to reason about what happens when we add points to the dataset.

Beyond their use in our formal arguments, these formulas show how our goodness definition in \cref{def:goodness} is essentially the ``right'' one to analyze stability. The leverage score and the magnitude of the residual \textit{exactly} determine the sensitivity of the least-squares solution to adding or removing that data point. 
To see this in more detail, consider the effect of dropping a point from a typical dataset:
\begin{align}
    \norm[\big]{(X^\T X)^{1/2} (\ols - \beta_{\mathrm{ols}(-j)})}^2 
        &= \norm*{(X^\T X)^{1/2} \cdot \frac{(X^\T X)^{-1} x_j}{1-h_j}\cdot (y_j - \ip{x_j}{\ols})}^2 \\
        &= \frac{h_j\cdot e_j^2}{(1- h_j)^2} = \Delta.
\end{align}
Arguing heuristically for now, if removing a point changes the OLS solution by $\Delta$, to ensure privacy one must ensure noise of magnitude \emph{at least} $\Delta$. It is impossible to do any better. Note that by working with $(\leverage,\residual)$-good sets we can guarantee that the noise we add for privacy,  $\cN(0,c^2 (X^\T X)^{-1})$ where $c^2 \ge \leverage\residual^2 \cdot \frac{\log 1/\delta}{\eps^2}$,  has magnitude roughly $\Delta$. 
This insight shows our accuracy guarantees are sharp.

\subsection{Notation}

We use \(\sbr{n}\) to denote the set \(\set{1, \ldots, n}\) and \(\bbN = \set{1, 2, \ldots}\).
For a vector \(v\in \bbR^n\) its support is \(\supp{v} = \set{i \in \sbr{n} \mid v_i \neq 0}\).
If we have a set \(S\subseteq \sbr{n}\), then \(\vsmask{v}{S}\in \RR^n\) has \(\del{\vsmask{v}{S}}_i = v_i\) for \(i \in S\) and \(\del{\vsmask{v}{S}}_i = 0\) otherwise.
Also we define \(\overline{S} = \sbr{n}\setminus S\).
We use \(\norm{v} \defeq \norm{v}_2\) and \(\norm{v}_{\cov} \defeq \norm{\cov^{1/2}v}\).
If \(M \in \RR^{n \times n}\) is a matrix, then \(\norm{M}_2\) denotes its spectral norm.

\section{Analysis of Greedy Residual Thresholding}
\label{sec:greedy_analysis}

\begin{algorithm2e}
\SetAlgoLined
\SetKwInOut{Input}{input}
\SetKwInOut{Require}{require}
\Input{dataset \(X, y\); outlier threshold $R$; starting weights $w$}
\BlankLine
\While{$\mathtt{TRUE}$}{
    $\beta_w \gets \weightedOLS(X,y,w)$\;
    $i^* \gets \argmax_{i\in \supp{w}} \abs[\big]{y_i - x_i^\T \beta_w}$ \;
    \eIf{$\abs{y_{i^*}- x_{i^*}^\T \beta_w} > R$}{
        $w_{i^*} \gets 0$\;
    }{
        \KwRet $w$\;
    }
}
\caption{\resthresh}\label{alg:residual_thresholding}
\end{algorithm2e}

In this section we establish the key properties of our greedy residual-thresholding algorithm.
This analysis contains the bulk of the technical novelty in our work.
The main result is \cref{claim:intertwining}, the ``intertwining'' property that relates the outputs of \resthresh on adjacent datasets.

Since we will be dealing extensively with weighted sets from now on, we expand the definition of good sets in \cref{def:goodness_set} to vectors of weights.
\begin{restatable}[$(\leverage,\residual)$-goodness, weighted]{defnn}{goodnessdef}
\label[definition]{def:goodness}
    Fix a dataset $(X,y)\in \bbR^{n\times d}\times \bbR^n$ and parameters $\leverage,\residual>0$.
    A vector $w\in [0,1]^n$ is \emph{$(\leverage,\residual)$-good for $(X,y)$} if, denoting $W=\mathrm{diag}(w)$, $X^\T W X$ is invertible and the following two conditions hold for all $i\in \mathrm{supp}(w)$.
    \begin{enumerate}
        \item Bounded leverage: $x_i^\T (X^\T W X)^{-1} x_i \le \leverage$.\label{def:goodness:leverage}
        \item Bounded residuals: $\abs[\big]{\ip{x_i}{\beta_w} - y_i} \le \residual$, where $\beta_w = (X^\T WX)^{-1}X^\T W y$.\label{def:goodness:residual}
    \end{enumerate}
    Furthermore, we will say that \(w\) is \(\del{L,\infty}\)-good for \(\del{X, y}\) if (\ref{def:goodness:leverage}) holds, but not (\ref{def:goodness:residual}).
\end{restatable}

\subsection{Stability and Goodness for Ordinary Least Squares}\label{sec:manipulating_regression}

Our analyses rely on how goodness is affected when adding and or removing mass.
As discussed in \cref{sec:tech}, closed-form expressions characterize the effects of removing a single point \citep{mendenhall2003second,belsley2005regression,huber2011robust}.
The following claim generalizes these results to removing multiple weighted points, or adding weight to points already included in the regression.
In addition, it shows how these results interact with goodness.
We defer the proof to \cref{sec:deferred_proofs}.
\begin{claim}[Changing Weight Within Support]\label{claim:final:goodness_after_removal}
    Let $w, w'\in[0,1]^n$ satisfy
        $\supp{w'} \subseteq \supp{w}$ and $\norm{w-w'}_1 \leverage \le \frac{1}{2}$. If w is $(L, \infty)$-good for $(X,y)$, then for all $i \in \supp{w}$,
    \begin{align}
        x_i^\T (X^\T \diag(w') X)^{-1} x_i &\le (1 + 2\leverage\norm{w-w'}_1)\,\leverage. \label{eq:removal_leverage}
    \end{align}
If, in addition to the previous conditions, it also holds that $w$ is $(L,R)$-good for $(X,y)$, then
    \begin{align}
                \abs{x_i^\T \beta_w - x_i^\T \beta_{w'}} &\le 2 \norm{w - w'}_1 \leverage \residual.  \label{eq:removal_prediction}
    \end{align}
    In particular, since $\supp{w'}\subseteq \supp{w}$,  Equations~\eqref{eq:removal_leverage} and \eqref{eq:removal_prediction} apply to all $i\in \supp{w'}$.
    Consequently, $w'$ is $(\eta\leverage,\eta\residual)$-good for $(X,y)$, where $\eta = 1 + 2\leverage\norm{w -w'}_1$.
\end{claim}

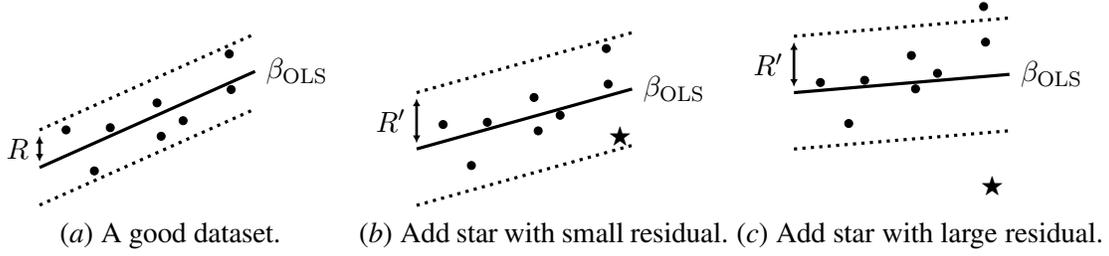
\begin{figure}
    \tikzset{>={Latex[width=3pt,length=3pt]}}
    \centering
    \subfigure[A good dataset.][b]{
        \begin{tikzpicture}[x=0.25cm,y=0.25cm]
        \draw [line width=1.2pt,dotted] (-5.7,0.8529905664600541)-- (5.76,5.97009166171396);
        \draw [line width=1.2pt] (5.76,3.970091661713961) node[right] {\(\beta_{\mathrm{OLS}}\)} -- (-5.7,-1.1470094335399454);
        \draw [line width=1.2pt,dotted] (-5.7,-3.1470094335399446)-- (5.76,1.9700916617139617);
        \draw [line width=1pt,<->] (-5.7,-1.1470094335399454 + 0.3)-- node[left] {\(R\)} (-5.7,0.8529905664600541-0.3);
        \draw [fill=black] (4.486088794926014,3.0042283298097217) circle (1.5pt);
        \draw [fill=black] (4.38460887949261,4.898520084566592) circle (1.5pt);
        \draw [fill=black] (-4.291923890063406,0.8562367864693408) circle (1.5pt);
        \draw [fill=black] (1.9321775898520206,1.3467230443974592) circle (1.5pt);
        \draw [fill=black] (0.5452854122621696,2.293868921775895) circle (1.5pt);
        \draw [fill=black] (-1.9409725158562197,0.9746300211416452) circle (1.5pt);
        \draw [fill=black] (0.7651585623678777,0.5179704016913281) circle (1.5pt);
        \draw [fill=black] (-2.7866384778012527,-1.3255813953488411) circle (1.5pt);
        \end{tikzpicture}
        \label{fig:add_wt_base}
    } %
    \subfigure[
        Add star with small residual.
    ][b]{
        \begin{tikzpicture}[x=0.25cm,y=0.25cm]
        \draw [line width=1.2pt,dotted] (-5.7,2.553412896902111)-- (5.76,5.745504091665829);
        \draw [line width=1.2pt] (5.76,2.745504091665829) node[right] {\(\beta_{\mathrm{OLS}}\)} -- (-5.7,-0.4465871030978887);
        \draw [line width=1.2pt,dotted] (-5.7,-3.446587103097889)-- (5.76,-0.25449590833417085);
        \draw [line width=1pt,<->] (-5.7,-0.4465871030978887 + 0.3)-- node[left] {\(R'\)} (-5.7,2.553412896902111 - 0.3);
        \draw [fill=black] (4.486088794926014,3.0042283298097217) circle (1.5pt);
        \draw [fill=black] (4.38460887949261,4.898520084566592) circle (1.5pt);
        \draw [fill=black] (-4.291923890063406,0.8562367864693408) circle (1.5pt);
        \draw [fill=black] (1.9321775898520206,1.3467230443974592) circle (1.5pt);
        \draw [fill=black] (0.5452854122621696,2.293868921775895) circle (1.5pt);
        \draw [fill=black] (-1.9409725158562197,0.9746300211416452) circle (1.5pt);
        \draw [fill=black] (0.7651585623678777,0.5179704016913281) circle (1.5pt);
        \draw [fill=black] (-2.7866384778012527,-1.3255813953488411) circle (1.5pt);
        \node at (5.12879492600424,0.2) {\footnotesize\(\bigstar\)};
        \end{tikzpicture}
        \label{fig:add_wt_good}
    } %
    \subfigure[
        Add star with large residual.
    ][b]{
        \begin{tikzpicture}[x=0.25cm,y=0.25cm]
        \draw [line width=1.2pt,dotted] (-5.7,3.3158948285089327)-- (5.76,4.289459181879822);
        \draw [line width=1.2pt] (5.76,1.2894591818798218)  node[right] {\(\beta_{\mathrm{OLS}}\)} -- (-5.7,0.31589482850893275);
        \draw [line width=1.2pt,dotted] (-5.7,-2.6841051714910673)-- (5.76,-1.7105408181201782);
        \draw [line width=1pt,<->] (-5.7,0.31589482850893275 + 0.3)-- node[left] {\(R'\)} (-5.7,3.3158948285089327 - 0.3);
        \draw [fill=black] (4.486088794926014,3.0042283298097217) circle (1.5pt);
        \draw [fill=black] (4.38460887949261,4.898520084566592) circle (1.5pt);
        \draw [fill=black] (-4.291923890063406,0.8562367864693408) circle (1.5pt);
        \draw [fill=black] (1.9321775898520206,1.3467230443974592) circle (1.5pt);
        \draw [fill=black] (0.5452854122621696,2.293868921775895) circle (1.5pt);
        \draw [fill=black] (-1.9409725158562197,0.9746300211416452) circle (1.5pt);
        \draw [fill=black] (0.7651585623678777,0.5179704016913281) circle (1.5pt);
        \draw [fill=black] (-2.7866384778012527,-1.3255813953488411) circle (1.5pt);
        \node at (4.841268498942929,-4.691331923890071) {\footnotesize\(\bigstar\)};
        \end{tikzpicture}
        \label{fig:add_wt_bad}
    } 
    \caption{
        We illustrate our analysis of adding a new observation (star) to an $(\leverage,\residual)$-good dataset (\emph{a}).
        In (\emph{b}), the added star is close to the original regression line.
        The new largest residual may be greater than $\residual$ but is less than $\residual'$. 
        In (\emph{c}), we instead add a significant outlier. 
        Multiple points may have residuals larger than $\residual'$, but the largest belongs to the star.
        In this case, residual thresholding discards the star and recovers the original dataset.
    } 
    \label{fig:add_wt}
\end{figure}

We next present a claim about adding a point to existing good weights: either the 
expanded weights are good or the new point has a large residual (in which case our greedy algorithm, presented later, will identify it).
We illustrate these cases in \cref{fig:add_wt}.
Such a claim also holds when we add sets of points.

Mathematically, this proof contains little innovation beyond \cref{claim:final:goodness_after_removal}.
However, it provides a key conceptual bridge.
We see that it connects directly to our greedy algorithm, which removes large residuals.
\begin{claim}[Adding Weight Outside Support]\label{claim:add_weight_outside_support}
    Let $w'\in[0,1]^n$ be an $(\leverage,\residual)$-good vector for a dataset $(X,y)$ and 
    let $v\in[0,1]^n$ be a vector such that $\supp{w'}\cap\supp{v}=\emptyset$.
    Define $w=w'+v$ and $\eta = 1 + 8\norm{v}_1\leverage$. Assume the following two conditions hold:

    \begin{enumerate}
        \item The matrix $X^\T \diag(w) X$ is invertible and for all $j \in \supp{w}$, $$ x_j^\T (X^\T \diag(w) X)^{-1}x_j\le 2\leverage.$$
        \item The weights $v$ satisfy $\norm{v}_1 \cdot \leverage\le \frac 1 8$  
    \end{enumerate}
    If $\max_{i \in \supp{w}} \abs[\big]{y_i - x_i^\T \beta_{w}} > \eta \residual$, then $\argmax_{i \in \supp{w}} \abs[\big]{y_i - x_i^\T \beta_{w}} \subseteq \supp{v}$.
\end{claim}
\begin{proof}
    We prove the contrapositive: if there exists $j^* \in \argmax_{i\in \supp{w}} \abs{y_i - x_i^\T \beta_{w}}$ with $j^*\notin \supp{v}$, then for all $i  \in \supp{w}$, $\abs{y_i - x_i^\T \beta_{w}}\le \eta \residual$.

    Note that since $\supp{w} = \supp{w'} \cup \supp{v}$ and $\supp{w'}\cap\supp{v}=\emptyset$, $j^*\notin \supp{v}$ implies $j^*\in\supp{w'}$. 
    We first produce a \emph{lower bound} on the $j^*$ residual under $w'$.
    By the triangle inequality, 
    \begin{align}
        \abs[\big]{e_{j^*}'} = \abs[\big]{y_{j^*} - x_{j^*}^\T \beta_{w'}}
            &= \abs[\big]{y_{j^*} - x_{j^*}^\T \beta_w + x_{j^*}^\T \beta - x_{j^*}^\T \beta_{w'}} \\
            &\ge \abs[\big]{y_{j^*} - x_{j^*}^\T \beta_w} - \abs[\big]{x_{j^*}^\T \beta - x_{j^*}^\T \beta_{w'}} \\
            &= \abs[\big]{e_{j^*}} - \abs[\big]{x_{j^*}^\T \beta_w - x_{j^*}^\T \beta_{w'}}. \label{eq:residuals_reverse_triangle}
    \end{align}
    Note that by assumption, $w$ is $(2L, |e_{j^*}|)$-good for $(X,y)$. Since $\supp{w'}\subseteq \supp{w}$, and $\norm{w-w'}_1=\norm{v}_1 \le \frac{1}{8\leverage}$, we can apply 
    \cref{claim:final:goodness_after_removal} to get that $\abs{x_{j^*}^\T \beta_w - x_{j^*}^\T \beta_{w'}} \leq 2 \norm{w-w'}_1 (2L) |e_{j^*}|$. Using this upper bound, we get that:
    \begin{align}
            \abs[\big]{e_{j^*}'}  \geq \abs[\big]{e_{j^*}} -  4 \norm[\big]{w-w'}_1 \leverage \abs[\big]{e_{j^*}}
    \end{align}
        To complete the proof, we use the upper bound $\abs{e_j'}\le \residual$, which holds by the assumption that $j\in \supp{w'}$ (and that $w'$ is $(L,R)$-good). Rearranging our previous inequality, we get that for all $i \in \supp{w}$, 
        \begin{align}
         \abs{y_i - x_i^\T \beta_{w}}\le   \abs{e_{j^*}} \le \frac{\abs{e_{j^*}'}}{1-4\leverage\norm{w-w'}_1}
                \le \paren{1 + 8 \leverage\norm{w-w'}_1} \cdot \residual, 
        \end{align}
    where we have used the inequality $(1-z)^{-1} \leq 1 + 2z$, which holds for all $z \in (0,1/2]$. 
\end{proof}

\subsection{Guarantees for Leverage Filtering}

\resthresh receives as input a vector $w$, the ``starting weights,'' and iteratively zeros out any weights corresponding to residual outliers, 
recomputing the weighted OLS solution as it goes.
These starting weights $w$ will come from the \stablecovariance subroutine of BHS, which filters out high-leverage outliers.
The exact algorithm we use differs superficially from the version in BHS, who use it for covariance estimation and call it ``Stable Covariance.''
Our application only needs its properties as a leverage-filtering procedure.
We give a complete description of our variant in  \cref{sec:statement_of_stablecov}.

The filtering algorithm has ``goodness'' guarantees (when the score is modest, many points receive weight and no point with weight has high leverage), utility guarantees (on outlier-free data, all points receive full weight), and stability guarantees on adjacent datasets (the score is low-sensitivity and the weights are stable).
We now give the formal statement.
    
\newcounter{stabcovcount}
\begin{theorem}[Guarantees for \stablecovariance, \cite{brown2023fast}]\label{thm:stable_cov_guarantees}
    There is a deterministic algorithm \stablecovariance receiving as input a list of vectors $X\in\bbR^{n\times d}$, a leverage threshold $\leverage$, and a discretization parameter $k\in\bbZ$ and returning as output an integer $\score$ and a vector $w\in [0, 1]^n$.
    Let $W = \diag(w)$.
    Assume $kL\le 1$.
    If $\score<k$ the following hold.
    \begin{enumerate}
        \item $\norm{w}_1\ge n-k$.
            As a consequence, $\abs{\supp{w}}\ge n - k$.
        \item For all $i\in \supp{w}$, we have $x_i^\T (X^\T W X)^{-1} x_i \le \leverage$.
        \setcounter{stabcovcount}{\value{enumi}}
    \end{enumerate}
    On ``outlier-free'' data as defined below, the algorithm's output is as follows.
    \begin{enumerate}
        \setcounter{enumi}{\value{stabcovcount}}
        \item If $x_i^\T (X^\T X)^{-1} x_i \le \leverage/2e^2$ for all $i\in[n]$ then $\score=0$ and $w=\vec 1$. 
        \setcounter{stabcovcount}{\value{enumi}}
    \end{enumerate}
    To present the stability guarantees, let $X$ and $X'$ be datasets that differ in one entry.
    For any values of $k$ and $\leverage$, consider
    \begin{align}
        \score, w &\gets \stablecovariance(X,\leverage,k) \\
        \score', w' &\gets \stablecovariance(X',\leverage,k). 
    \end{align}
    We have the following sensitivity bounds. 
    \begin{enumerate}
        \setcounter{enumi}{\value{stabcovcount}}
        \item $\abs{\score -\score'}\le 2$.
        \item If $\score,\score'<k$ then $\norm{w-w'}_1\le 2$.
    \end{enumerate}
\end{theorem}

\subsection{Properties of \resthresh}

The first claim we prove says that, when we run \stablecovariance followed immediately by \resthresh, the returned weights are good.
\begin{claim}\label{lemma:resthresh_returns_good}
    Let $(X,y)$ be a dataset, $k\in \bbN$ be a discretization parameter, and 
    $\leverage,\residual>0$ be outlier thresholds. 
    Assume $k\leverage\le 1/2$.
    Consider the outputs of the following calls, where the latter uses the output of the former:
    \begin{align}
        \score, w &\gets \stablecovariance(X,\leverage, k) \;, \\
        u &\gets \resthresh(X,y, \residual, w).
    \end{align}
    If $\score < k$ and $\norm{u}_1\ge n - k$ 
    then $u$ is $(2\leverage, \residual)$-good for $X,y$.
\end{claim}
\begin{proof}
    By the guarantees of \stablecovariance, \cref{thm:stable_cov_guarantees}, when $\score<k$ the weights $w$ give us a bound of $\leverage$ on leverage. That is, for all $i \in \supp{w}$, 
    \begin{align}
        x_i^\T(X^\T W X)^{-1} x_i \leq L. 
    \end{align}
Furthermore, since \resthresh only alters $w$ by setting some entries to 0, we have that $\norm{u}_1 = \norm{w}_1 - \norm{w-u}_1$. Using the assumption that $\norm{u}_1 \geq n-k$ and the trivial bound $\norm{w} \leq n$, we get that $\norm{w-u}_1\le k$.
    Thus, setting $U  =\diag(u)$, by the first part of †\cref{claim:final:goodness_after_removal}, we have for all $i\in \supp{u}$ that
    \begin{align}
         x_i^\T (X^\T U X)^{-1} x_i &\le \del[\big]{1+ 2\leverage\norm{w-u}_1}\cdot\leverage 
        \le 2\leverage,
    \end{align}
    where we used the assumption that $\leverage k \le 1/2$.
    Since \resthresh only returns a vector when the largest absolute residual is no greater than $\residual$, we are done.
\end{proof}

Our next claim relates the runs of residual thresholding on adjacent datasets at nearby residual thresholds.
This is the main result about our thresholding procedure.

\begin{claim}[Intertwining]\label{claim:intertwining}
    Let $(X,y)$ and $(X',y')$ be adjacent datasets that differ on index $i^*$.
    Let $k\in \bbN$ be a discretization parameter.
    Let $\leverage,\residual,$ and $\residual' >0$ be any  outlier thresholds such that 
    $k\leverage\le \frac{1}{96}$ and $\residual'\ge \exp\del{108kL}\residual$.
    Consider the outputs of the following calls:
    \begin{align}
        w, \score &\gets \stablecovariance(X,\leverage, k) \\
        w', \score' &\gets \stablecovariance(X',\leverage, k),
    \end{align}
    which we feed into:
    \begin{align}
        u &\gets \resthresh(X,y, \residual, w) \\
        u' &\gets \resthresh(X',y', \residual', w').
    \end{align}
    Define \(I = \supp{u}\cap\supp{w'} \setminus \set{i^*}\).
    If $\score,\score' < k$ and \(\norm{u}_1 \ge n-k\) then 
    $I\subseteq \supp{u'}$.
\end{claim}

As we will see in \cref{sec:stable_analysis}, where \stableOLS uses \resthresh to obtain stable weights, many indices of the weights are easily accounted for.
This includes \(i^*\), which can be handled as a special case, as well as \(\supp{w} \setminus \supp{w'}\) and \(\supp{w'} \setminus \supp{w}\) whose stability is established by \cref{thm:stable_cov_guarantees}.
Ignoring those cases for now, we wish to show that any point that is not filtered under $(X,y)$ will also not be filtered under $(X',y')$ provided that the threshold used to filter $(X',y')$ is sufficiently large.
We illustrate these cases in \cref{fig:intertwining}.
Now we are ready to state the proof of \cref{claim:intertwining}.

\definecolor{ibmblue}{HTML}{648FFF}
\definecolor{ibmmagenta}{HTML}{DC267F}
\definecolor{ibmorange}{HTML}{FE6100}
\definecolor{ibmyellow}{HTML}{FFB000}

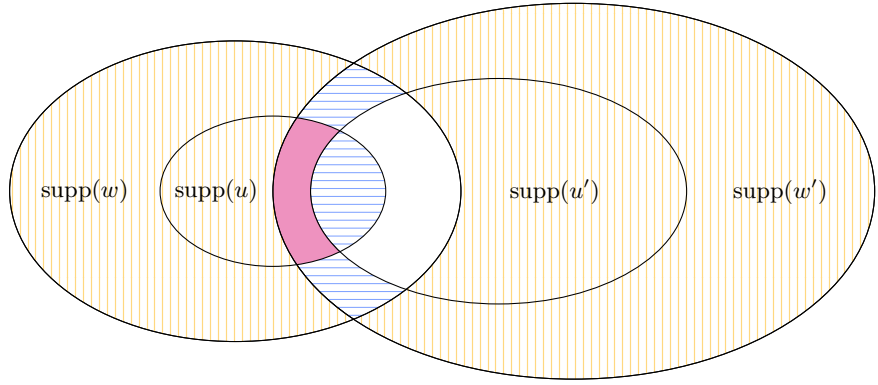
\begin{figure}
    \centering
    \begin{tikzpicture}[x=0.5cm,y=0.5cm]
        \newcommand{\CircleU}{(0,0) ellipse (3 and 2)}
        \newcommand{\CircleW}{(-1,0) ellipse (6 and 4)}
        \newcommand{\CircleUp}{(6,0) ellipse (5 and 3)}
        \newcommand{\CircleWp}{(8,0) ellipse (8 and 5)}

        \begin{scope}
            \clip \CircleW;
            \clip \CircleWp;
            \draw[pattern=horizontal lines, pattern color=ibmblue!80] \CircleWp;
        \end{scope}

        \begin{scope}
            \clip \CircleWp;
            \fill[ibmmagenta!50] \CircleU; 
        \end{scope}
        \fill[white] \CircleUp;
        
        \draw [pattern=vertical lines, even odd rule, pattern color=ibmyellow!50] \CircleW \CircleWp;
        \begin{scope}
            \clip \CircleU;
            \clip \CircleUp;
            \draw[pattern=horizontal lines, pattern color=ibmblue!80] \CircleUp;
        \end{scope}
        
        \draw \CircleU;
        \draw \CircleW;
        \draw \CircleUp;
        \draw \CircleWp;
        \node at (7.5, 0) {\footnotesize\(\supp{u'}\)};
        \node at (-1.5, 0) {\footnotesize\(\supp{u}\)};
        \node at (-5, 0) {\footnotesize\(\supp{w}\)};
        \node at (13.5, 0) {\footnotesize\(\supp{w'}\)};
    \end{tikzpicture}
    \caption{
    Graphical depiction of \cref{claim:intertwining}'s ``intertwining.'' 
    Here $w$ represents the weights on dataset $(X,y)$ after leverage filtering and $u$ the weights on $(X,y)$ after residual filtering. 
    $w'$ and $u'$ represent the analogous weights on an adjacent dataset $(X',y')$.
    By \cref{thm:stable_cov_guarantees}'s guarantees for leverage filtering, the yellow, vertically hatched regions represent a small amount of weight.
    The blue, horizontally hatched regions represent identical outcomes after residual filtering (either kept on both $(X,y)$ and $(X',y')$ or discarded on both).
    The claim's main consequence is that only {\em one} index can fall in the magenta, solid region, which receives weight under $u$ and $w'$ but not $u'$.
    This is $i^*$, the index that differs between $(X,y)$ and $(X',y')$.
    }
    \label{fig:intertwining}
\end{figure}

\medskip
\begin{proof}
    Our first goal will be to show that \(\vsmask{w'}{I}\) is sufficiently good.
    (Recall our notation: $\vsmask{w'}{I}\in [0,1]^n$ takes the value $w_i'$ for $i\in I$ and 0 elsewhere.)
    First, we see that \(u\) is \(\del{2\leverage, \residual}\)-good for \(\del{X, y}\) by noting that \(\score < k\), \(\norm{u}_1 \ge n-k\), and \(k\leverage \le 1/2\) and applying \cref{lemma:resthresh_returns_good}.
    Next, we show that \(\vsmask{w'}{I}\) is close to \(u\) in $\ell_1$ distance.
    In particular, by definition of \resthresh, if \(i \in \supp{u}\) then \(u_i = w_i\). Hence,
    \begin{align}
    \norm*{u - \vsmask{w'}{I}}_1 
    &= \sum_{i=1}^n \abs*{u_i - \del*{\vsmask{w'}{I}}_i}\\
    &= \sum_{i\in \supp{u}} \abs*{u_i - \del*{\vsmask{w'}{I}}_i} + \sum_{i\not\in \supp{u}} \abs*{u_i - \del*{\vsmask{w'}{I}}_i}
\end{align}
If $i \notin \supp{u}$, then $u_i = 0$ and hence $\del*{\vsmask{w'}{I}}_i = 0$ since $\vsmask{w'}{I}$ is by definition only nonzero outside the support of $u$. Hence, the second term in the last equation is 0. Moving on, by definition of the set $I$,
\begin{align}    
   \norm*{u - \vsmask{w'}{I}}_1  &= \sum_{i\in \supp{u}} \abs[\big]{w_i - \del*{\vsmask{w'}{I}}_i}\\ 
    &= \sum_{i\in \supp{u}} \abs[\big]{w_i - \del*{\vsmask{w'}{\supp{w'} \setminus \set{i^*}}}_i}\\
    &\le \abs[\big]{w_{i^*} - \del*{\vsmask{w'}{\supp{w'} \setminus \set{i^*}}}_{i^*}} + \sum_{i\in \supp{u}} \abs[\big]{w_i - \del*{\vsmask{w'}{\supp{w'}}}_{i}}\\
    & = \abs{w_{i^*} - 0} + \sum_{i\in \supp{u}} \abs{w_i - w'_i}\\
    &\le \abs{w_{i^*}} + \norm{w - w'}_1\\
    &\le 3\;.
    \end{align}
The the last inequality follows from the last part of \cref{thm:stable_cov_guarantees}. 

    Since \(I \subseteq \supp{u}\), and \(L \le kL\le 1/12\) by assumption, it holds that $ \norm*{u - \vsmask{w'}{I}}_1 \leq 3 \leq 1/(2L)$, and we can apply \cref{claim:final:goodness_after_removal} to show that \(\vsmask{w'}{I}\) is \(\del{2\eta_1 L, \eta_1 \residual}\)-good for \(\del{X, y}\) where \(\eta_1 = 1 + 12L\).
    Furthermore, \(\vsmask{w'}{I}\) is \(\del{2\eta_1 L, \eta_1 \residual}\)-good for \(\del{X', y'}\) because \(\del{\vsmask{w'}{I}}_{i^*} = 0\).

    Now, we will show that during the execution of \(\resthresh(X',y', \residual', w')\) we will never discard any \(i \in I\).
    Let \(w'^{\del{j}}\) denote the weights obtained in the \(j\)\textsuperscript{th} iteration of the while-loop out of \(m\) total iterations, such that \(w'^{\del{0}} = w'\) and \(w'^{\del{m}} = u'\).
    We proceed to show the loop invariant \(\vsmask{w'^{\del{j}}}{I} = \vsmask{w'}{I}\) for \(j \in \set{0, \ldots, m}\).
    Since \(w' = w'^{\del{0}}\), the invariant holds initially.
    
    In each iteration, by the loop invariant we can decompose \(w'^{\del{j-1}}\) as
    \[w'^{\del{j-1}} = \vsmask{w'^{\del{j-1}}}{I} + \vsmask{w'^{\del{j-1}}}{\setcompl{I}} = \vsmask{w'}{I} + \vsmask{w'^{\del{j-1}}}{\setcompl{I}}.\]
    Now, we note that we have \(\abs{\supp{u}} \ge \norm{u}_1 \ge n-k\) by assumption and since \(\score' < k\), we also have \(\abs{\supp{w'}} \ge n - k\) by \cref{thm:stable_cov_guarantees}.
    Therefore by inclusion-exclusion, \(\abs*{\setcompl{I}} \le 2k + 1\) and so \(\norm*{\vsmask{w'^{\del{j-1}}}{\setcompl{I}}} \le 2k + 1\).
    Now since \(96kL \le 1\) by assumption, we note that
    \begin{align}
        2\norm[\big]{\vsmask{w'^{\del{j-1}}}{\setcompl{I}}}\eta_1 L\le 2\del{2k+1}\eta_1 L \le 2\del{2k+1}\del{1 + 12L} L \le 12kL \le \frac{1}{8}
    \end{align}
    and for \(\eta_2 = 1 + 16\del{2k + 1}\eta_1 L\),
    \begin{align}
        \eta_2\eta_1R 
        &= \del{1 + 16\del{2k + 1}\del{1 + 12L} L}\del{1 + 12L}R\\
        &\le \del{1 + 96k L}\del{1 + 12L}R\\
        &\le \exp\del{96k L}\exp\del{12L}R\\
        &\le \exp\del{108k L}R\\
        &\le R'.
    \end{align}
    Thus by the \(\del{2\eta_1 L, \eta_1 \residual}\)-goodness of \(\vsmask{w'}{I}\) and \cref{claim:add_weight_outside_support}, if \(\max \abs{y_i - x_i^\T \beta_{w'^{\del{j-1}}}} \ge \residual'\) then \( \argmax \abs{y_i - x_i^\T \beta_{w'^{\del{j-1}}}} \not\in I\) and so \(\vsmask{w'^{\del{j}}}{I} = \vsmask{w'}{I}\).
    Finally, it follows that \(\vsmask{u'}{I} = \vsmask{w'}{I}\).
    Therefore, since \(I \subseteq \supp{w'}\), we have \(I \subseteq \supp{u'}\).
\end{proof}

We now observe that our greedy residual thresholding subroutine only removes more points when run with smaller thresholds.
We now state this simple fact for future reference.
\begin{observation}\label{obs:greedy_nesting}
    Fix a dataset $(X,y)$ and starting weights $w$.
    For outlier thresholds $\residual\le \residual'$, consider running
    \begin{align}
        u &\gets \resthresh(X,y,\residual,w) \\
        u' &\gets \resthresh(X,y,\residual',w). 
    \end{align}
    For all $i\in [n]$, $u_i \le u_i'$.
\end{observation}
    
\section{Analysis of \stableOLS}
\label{sec:stable_analysis}

\begin{algorithm2e}
\SetAlgoLined
\SetKwInOut{Input}{input}
\SetKwInOut{Require}{require}
\Input{dataset \(X, y\); base outlier thresholds $\leverage_0,\residual_0$; weights $w$; discretization parameter $k$}
\BlankLine
$\forall j\in [2k], \residual_j \gets (\Rjump)^j \cdot \residual_0$\;
\For{$j\in \{0,\ldots,2k\}$}{
    $u^{(j)} \gets \resthresh(X,y,\residual_j, w)$\;
    $\score^{(j)}\gets \min\{k, n - \norm{u^{(j)}}_1 + j$\}\;
}
$\score\gets \min_{j\in \{0,\ldots,k\}} \score^{(j)}$\;
$v\gets \frac 1 k \sum_{j=k+1}^{2k} u^{(j)}$\;
\KwRet $\score, v$
\caption{\stableOLS}\label{alg:stable_OLS}
\end{algorithm2e}

In this section, we prove the stability guarantees for \cref{alg:stable_OLS}, our new regression estimator.
\cref{alg:stable_OLS} repeatedly calls \cref{alg:residual_thresholding}, \resthresh, over a range of slowly increasing outlier thresholds.
These thresholds are indexed by a number $j\in \{0,1,\ldots,2k\}$, where $k$ is a discretization parameter.
(Later, we connect this discretization to the privacy parameters, setting $k\approx \log(1/\delta)/\eps$.)
The key lemma used in these proofs is \cref{claim:intertwining}, which relates the weights found on a dataset $(X,y)$ at level $j$ to the weights found on an adjacent dataset $(X',y')$ at level $j+1$.

We start by showing that the $\score$ value and  the weight vector, $v$, returned by \cref{alg:stable_OLS} are low-sensitivity.
\begin{claim}[Score is Low-Sensitivity]\label{claim:score_low_sensitivity}
    Let $(X,y)$ and $(X',y')$ be adjacent datasets. 
    Fix outlier thresholds $\leverage,\residual$ and discretization parameter $k$.
    Assume \(kL \le \frac{1}{96}\).
    Let 
    \begin{align}
        \score_1, w &\gets \stablecovariance(X,\leverage,k) \\
        \score_1', w' &\gets \stablecovariance(X',\leverage,k)
    \end{align}
    and
    \begin{align}
        \score_2, v &\gets \stableOLS(X,y,w,\leverage,\residual,k) \\
        \score_2', v' &\gets \stableOLS(X',y',w',\leverage,\residual,k).
    \end{align}
    If $\score_1,\score_1' < k$, then $\abs{\score_2-\score_2'} \le 4$.
\end{claim}
\begin{proof} 
 We observe that all $\score$ variables are at most $k$ by construction.
    Without loss of generality, assume \(\score_2 \le \score_2'\).
    First, we consider the case when \(\score_2 = k\). In this setting, since \(\score_2' \le k\) we must have \(\score_2 = \score_2'\), hence \(\abs{\score_2 - \score_2'}=0\)   and we are done.
    
    Now, consider the case when \(\score_2 < k\).
    Then, by definition of the $\stableOLS$ algorithm, there must exist a \(j^* \in \set{0, \ldots, k}\) such that \(n - \norm{u^{\del{j^*}}}_1 + j^* = \score_2\), where $u^{\del{j^*}}$ are the weights returned by the $\resthresh$ subroutine (run within $\stableOLS$) at outlier threshold $R_{j^*}$.
    
    Let \(u = u^{\del{j^*}}\) and \(u' = \del{u'}^{\del{j^*+1}}\) denote the weights returned by \resthresh on dataset and outlier thresholds \(\del{X, y}, \residual_{j^*}\) and \(\del{X', y'}, \residual_{j^*+1}\) respectively.
    Defining \(I\) as in \cref{claim:intertwining}, we note that 
    \begin{align}
        \norm{u'}_1 \ge \norm*{\vsmask{u'}{I}}_1 = \norm*{u - \del{u - \vsmask{u'}{I}}}_1 \ge \norm*{u}_1 - \norm{u - \vsmask{u'}{I}}_1.
    \end{align}
    Now, seeking to bound the last term using \cref{claim:intertwining}, we note that
    \begin{enumerate}
        \item \(\residual_{j^*+1}/\residual_{j^*} \ge \exp\del{108kL}\) by the definition in \cref{alg:stable_OLS},
        \item \(kL \le \frac{1}{96}\) by assumption
        \item Since, $\norm{u^{(j^*)}}_1 = \norm{u}_1$ and $\score_2 = n - \norm{u^{(j^*)}}_1 + j^* < k$, it holds that $$\norm{u}_1 > n- k + j^* \ge n - k.$$ 
    \end{enumerate}
    Therefore, we can apply \cref{claim:intertwining}, which implies that \(\vsmask{u'}{I} = \vsmask{w'}{I}\).
    This gives
    \begin{align}
        \norm{u - \vsmask{u'}{I}}
        &= \norm*{\vsmask{w}{\supp{u}} - \vsmask{w'}{I}}\\
        &= \norm*{\vsmask{w}{\supp{u}} - \vsmask{w'}{\supp{u} \setminus \set{i^*}}}\\
        &\le \norm*{\vsmask{{w - w'}}{\supp{u}} } + 1\\
        &\le 3.
    \end{align}
    Finally, combining the previous results gives
    \begin{align}
        \score_2' 
        &\le n - \norm*{\del{u'}^{\del{j^*+1}}}_1 + \del{j^*+1}\\
        &\le n - \del{\norm{u^{\del{j^*}}}_1 - 3} + j^*+1\\
        &= \del{n - \norm{u^{\del{j^*}}}_1 + j^*} + 4\\
        &= \score_2 + 4.
    \end{align}
The first inequality in the calculation about holds by definition of $\stableOLS$. The second one uses our previous two calculations.
\end{proof}

\begin{claim}[Weights are Stable]\label{claim:weights_stable}
    Let $(X,y)$ and $(X',y')$ be adjacent datasets. 
    Fix outlier thresholds $\leverage,\residual$ and discretization parameter $k$.
    Assume \(kL \le \frac{1}{96}\).
    Let 
    \begin{align}
        \score_1, w &\gets \stablecovariance(X,\leverage,k) \\
        \score_1', w' &\gets \stablecovariance(X',\leverage,k)
    \end{align}
    and
    \begin{align}
        \score_2, v &\gets \stableOLS(X,y,w,\leverage,\residual,k) \\
        \score_2', v' &\gets \stableOLS(X',y',w',\leverage,\residual,k).
    \end{align}
    If $\score_1,\score_1',\score_2,\score_2' < k$, then $\norm{v - v'}_1\le 5$.
\end{claim}
\begin{proof}
    Consider the execution of \resthresh resulting in a weight vector \(u\).
    We observe that \resthresh receives weight vector $w$ as input and modifies it by setting a subset of the weights to zero. 
    Thus we can write the weight $u_i = w_i \cdot 1\{u_i \neq 0\}$.
    This (rather trivial) modification allows us to write the output of \stableOLS in terms of counts: letting $c_i = \sum_{j=k+1}^{2k} 1\{u_i^{(j)}\neq 0\}$, we have
    \begin{align}
        v_i = \frac 1 k \sum_{j=k+1}^{2k} u_i^{(j)} = \frac{w_i}{k} \sum_{j=k+1}^{2k} 1\{u_i^{(j)}\neq 0\} = \frac{w_i c_i}{k}.
    \end{align}
    Now note that by \cref{obs:greedy_nesting}, for \(j \in \set{k+1, \ldots, 2k}\) we can have \(u_i^{\del{j-1}} \neq 0\) only if \(u_i^{\del{j}} \neq 0\).
    This implies that \(u_i^{\del{2k-c_i}} \neq 0\) and  \(u_i^{\del{2k-c_i-1}} = 0\).
    
    Now, since \(\score_2 < k\), we know that there exists some \(j^* \in \set{0, \ldots, k}\) such that \(n - \norm*{u^{\del{j^*}}}_1 + j^* < k\). 
    Applying, \cref{obs:greedy_nesting} again, we see that \(\norm{u^{\del{j}}}_1 \ge \norm{u^{\del{j^*}}}_1 > n-k\) for all \(j \ge j^*\).
    From this, we can conclude that \(\abs{\set{i \mid c_i \neq k}} < k\).
    
    Now, define \(c'\) analogously as the counts under \(\del{X', y'}\) and note that all of the previous observations apply under \(\del{X', y'}\) as well.
    Consider some \(i \in \supp{w} \cap \supp{w'} \setminus\set{i^*} \) and suppose without loss of generality that \(c_i \ge c'_i\).
    Our goal will be to show that \(c'_i \le c_i \le c'_i + 1 \).
    If we have \(c_i = k\), then \(c_i = c'_i\), so we turn our attention to the case where \(c_i < k\).
    We know that \(u_i^{\del{2k-c_i}} \neq 0\).
    Now, seeking to show that \(\del{u_i'}^{\del{2k-c_i+1}} \neq 0\) using \cref{claim:intertwining}, we note that
    \begin{enumerate}
        \item \(u_i^{\del{2k-c_i}}\) and \(\del{u_i'}^{\del{2k-c_i+1}}\) were computed using outlier thresholds \(\residual_{2k-c_i}\) and \(\residual_{2k-c_i+1}\) which satisfy \(\residual_{2k-c_i+1}/\residual_{2k-c_i} \ge \exp\del{108kL}\) by the definition in \cref{alg:stable_OLS},
        \item we have \(kL \le \frac{1}{96}\) by assumption, and
        \item we have \(\norm[\big]{u_i^{\del{2k-c_i}}}_1 > n - k\) as we observed previously.
    \end{enumerate}
    Therefore we can apply \cref{claim:intertwining}, which implies that \(\del{u_i'}^{\del{2k-c_i+1}} \neq 0\).
    Recalling our previous observation, we obtain \(c_i \le c'_i + 1\) as desired.
    In summary, if \(i \in \supp{w} \cap \supp{w'} \setminus\set{i^*} \) then we can write \(c'_i = c_i + \Delta_i\) where \(\abs{\Delta_i} \le 1\).

    Now, define \(D = \set{i \in\supp{w} \cap \supp{w'} \setminus\set{i^*}  \mid \Delta_i \neq 0}\) and note that \(\abs{D} \le 2k\) since \(c_i\) and \(c'_i\) both contain at most \(k\) elements not equal to \(k\).

    Now, we are ready to complete the proof by noting that we can decompose the quantity we wish to bound into four terms
    \begin{align}
        k\norm{v - v'}_1 
        &= \abs*{c_{i^*}w_{i^*} - c'_{i^*}w'_{i^*}} + \sum_{\substack{i \in \supp{w}\\ i\not\in \supp{w'}\\ i \neq i^*}}\abs{c_iw_i} + \sum_{\substack{i \not\in \supp{w}\\ i\in \supp{w'}\\ i \neq i^*}} \abs{c'_iw'_i} + \sum_{\substack{i \in \supp{w}\\ i\in \supp{w'}\\ i \neq i^*}} \abs{c_iw_i - c'_iw'_i}
    \end{align}

    This is valid because \(\abs{c_iw_i - c'_iw'_i}\) appears exactly once on the right for each \(i\).
    Now we will consider each term separately.
    The first term is at most \(k\) because \(c_i, c'_i\) are bounded by \(k\) and \(w_i, w'_i\) are bounded by \(1\).
    The summands of the second and third terms can be rewritten as \(c_i\abs{w_i - w'_i}\) and \(c'_i\abs{w_i - w'_i}\) and are thus bounded by \(k\abs{w_i - w'_i}\) and \(k\abs{w_i - w'_i}\) respectively. 
    Now focusing on the last term, we have
    \begin{align}
        \sum_{\substack{i \in \supp{w}\\ i\in \supp{w'}\\ i \neq i^*}} \abs{c_iw_i - c'_iw'_i}
        &\le  \sum_{\substack{i \in \supp{w}\\ i\in \supp{w'}\\ i \neq i^*}} \del*{\abs{c_iw_i - c_iw'_i} + \abs{\Delta_iw'_i}}\\
        &\le \sum_{\substack{i \in \supp{w}\\ i\in \supp{w'}\\ i \neq i^*}}\abs{c_iw_i - c_iw'_i} + \sum_{\substack{i \in \supp{w}\\ i\in \supp{w'}\\ i \neq i^*\\i \in D}}\abs{\Delta_iw'_i}\\
        &\le \sum_{\substack{i \in \supp{w}\\ i\in \supp{w'}\\ i \neq i^*}}k\abs{w_i - w'_i} + 2k.
    \end{align}
    Combining the bounds for each term, we have
    \begin{align}
        k\norm{v - v'}_1 & \le 3k + \sum_{\substack{i \in \supp{w}\\ i\not\in \supp{w'}\\ i \neq i^*}}k\abs{w_i - w'_i} + \sum_{\substack{i \not\in \supp{w}\\ i\in \supp{w'}\\ i \neq i^*}} k\abs{w_i - w'_i} + \sum_{\substack{i \in \supp{w}\\ i\in \supp{w'}\\ i \neq i^*}}k\abs{w_i - w'_i}\\
        &\le 3k + k\sum_{i}\abs{w_i - w'_i}\\
        &= 3k + k\norm{w - w'}_1\\
        &\le 5k
    \end{align}
    where the last line is an application of \cref{thm:stable_cov_guarantees}.
\end{proof}

The previous claim shows that the weights produced by \stableOLS on adjacent datasets are close in $\ell_1$ (when $\score$ is less than $k$).
In the next claim, we prove that the weights are good.
Note that this does not follow immediately from \cref{lemma:resthresh_returns_good}, which says that the weights returned by \resthresh are good.
Since \stableOLS returns an average of  the vectors returned by \resthresh, we have to argue that the average of good sets is good.

\begin{claim}[Weights are Good]\label{claim:weights-good}
    Fix a dataset $(X,y)$, outlier thresholds $\leverage,\residual$, and discretization parameter $k$.
    Assume $\leverage k \le \frac{1}{4}$.
    Consider the following calls:
    \begin{align}
        \score_1, w &\gets \stablecovariance(X, \leverage, k)\\ 
        \score_2, v &\gets \stableOLS(X,y, w, \leverage,\residual,k).
    \end{align}
    Then the vector $v$ is $(4\leverage,2\residual_{2k})$-good for $(X,y)$, where $\residual_{2k}= \exp(216 k^2 \leverage) \cdot \residual$.
\end{claim}
\begin{proof}
    \stableOLS calls \resthresh repeatedly, producing a vector $u^{(j)}$ for each residual threshold $\residual_j$.
    Recall from \cref{obs:greedy_nesting} that for all $i\le j$, since $R_i \leq R_j$ (within \stableOLS) we have that $u^{(i)}\le u^{(j)}$ elementwise.
    This implies that the support of $u^{(2k)}$ contains all other supports, including that of the average $v$.
    Additionally, since $\score_2<k$ (by construction within the algorithm), there exists some $j^*\in \{0,\ldots,k\}$ with $\norm{u^{(j^*)}}_1 \ge n-k$, so the same lower bound holds for all $j\ge k$.
    Together, these facts imply that $\norm{u^{(2k)} - u^{(j)}}_1 \le k$ for all $j\ge k$.
    Since the $\ell_1$ norm is convex and $v=\bbE_j[u^{(j)}]$, by Jensen's inequality we have $\norm{u^{(2k)}-v}_1 \le k$ as well.

    To finish the proof, we apply \cref{claim:final:goodness_after_removal}: since $\supp{v}\subseteq \supp{u^{(2k)}}$, $\norm{u^{(2k)} - v}_1 \le k$, and $u^{(2k)}$ is $(2\leverage, \residual_{2k})$-good, we conclude that $v$ is $(\leverage',\residual')$-good for
    \begin{align}
        \leverage' &\le (1 + 4\leverage k)\cdot 2\leverage \le 4\leverage \\
        \residual' &\le (1 + 4\leverage k)\cdot \residual_{2k} \le 2 R_{2k}.
    \end{align}
    Recalling that $\residual_{2k}= (\Rjump)^{2k} \cdot \residual_0$, we finish the proof.
\end{proof}

\section{Privacy Analysis of \cref{alg:private_OLS}}
\label{sec:main_analysis}

Our privacy analysis follows the blueprint established by \cite{brown2021covariance,duchi2023fast,brown2023fast}.
We use the well-known propose-test-release (PTR) framework of \cite{DworkL09} and first privately check (via our low-sensitivity $\score$) if it is safe to proceed.
If this check passes, we compute a vector of weights $v\in [0,1]^n$.
We use this vector to compute a weighted covariance $S_v$ and weighted least squares solution $\hat\beta_v$.
The output is then drawn from $\mathcal{N}(\hat\beta_v, c^2 S_v^{-1})$ for some appropriate constant $c$.

On adjacent datasets, we may compute different weights $v,v'$.
We know that, when the PTR checks pass, these vectors are close in $\ell_1$. 
The main work in this section, then, lies in connecting this stability of weights to stability of parameters, which in turn implies $\mathcal{N}(\hat\beta_v, c^2 S_v^{-1})\approx_{(\eps,\delta)} \mathcal{N}(\hat\beta_{v'}, c^2 S_{v'}^{-1})$.
Note that this is more complicated than the standard Gaussian mechanism, since both the shape and location of the noise change.

Before proving \cref{lemma:main_privacy}, our main privacy claim, we collect the necessary statements.
First, we recall the privacy check of BHS, which (in place of the standard Laplace-noise-and-threshold) simplifies our analysis.
\begin{claim}[PTR Mechanism]\label{claim:ptr_guarantees}
    Fix $0<\eps\le 1$, $0< \delta\le \frac{\eps}{10}$, and $0<\Delta$.
    There is an algorithm $\PTR: \bbR\to \braces{\pass,\fail}$ that satisfies the following conditions:
    \begin{enumerate}
        \item Let $\cU$ be a set and $g:\cU^n \to \bbR_{\ge 0}$ a function.
            If, for all $x,x'\in \cU^n$ that differ in one entry,  $\abs{g(x)-g(x')}\le \Delta$, then $\PTR(g(\cdot))$ is $(\eps,\delta)$-DP.
        \item $\PTR(0)=\pass$.
        \item For all $z\ge  \frac{\Delta \log 1/\delta}{\eps}+2\Delta$, $\PTR(z)=\fail$.
    \end{enumerate}
\end{claim}

The next claim relates bounded leverage, $\ell_1$ closeness, and covariance closeness.
This statement comes directly from BHS, Lemma 23; similar claims were used in \cite{brown2021covariance, duchi2023fast}.
We use the notation $\dpsd\del{\cov_1, \cov_2}$ to denote the maximum of  $\norm*{ \cov_1^{-1/2}\cov_2\cov_1^{-1/2}- \bbI }_{\tr}$ and $\norm*{ \cov_2^{-1/2}\cov_1\cov_2^{-1/2}- \bbI }_{\tr}$.
Recall that $\dpsd\del{\cov_1,\cov_2}=\dpsd\del{\cov_1^{-1},\cov_2^{-1}}$ (\cref{claim:dpsd-inverse}).
\newcommand{\psdjump}{\gamma}
\begin{claim}\label{claim:identifiability_covs}
    Let $\leverage \in (0,1)$ and let $X,X'\in \bbR^{n\times d}$ be adjacent (i.e., they differ in one out of $n$ rows).
    For vectors $v,w\in[0,1]^n$, let $\cov_v = X^\T \diag(v) X$ and $\cov_w = (X')^\T \diag(w) X'$.
    Suppose $v$ and $w$ both have bounded leverage: for all $i\in \supp{v}, x_i^\T \cov_v^{-1}x_i \le \leverage$ and for all $j\in \supp{w}, x_j^\T \cov_w^{-1} x_j\le \leverage$.
    Then $\cov_v$ and $\cov_w$ are positive definite and, 
    if $(1+\norm{v-w}_1)\leverage \le \frac 1 2$, satisfy
    \begin{align}
        \dpsd\del{\cov_v, \cov_w} \le 2\del[\big]{2+\norm{v-w}_1}\leverage.
    \end{align}
\end{claim}

An analogous claim says that, if we have two vectors $v$ and $v'$ that are $(\leverage,\residual)$-good on adjacent datasets and are close in $\ell_1$, then the regression parameters they induce are close.
We defer the proof to \cref{sec:deferred_proofs}, as similar claims appear in the robust statistics literature \citep{klivans2018efficient,bakshi2021robust}.
\begin{claim}\label[claim]{claim:parameter_stability_from_weight_stability_different_datasets}
    Let $(X,y)$ and $(X',y')$ be datasets differing in one entry.
    Let vector $v$ be $(\leverage,\residual)$-good for $(X,y)$ and let vector $w$ be $(\leverage,\residual)$-good for $(X',y')$.
    Set $V=\diag(v)$ and likewise $W$.
    Let $S_v=X^\T V X$, $\beta_v = (X^\T VX)^{-1}X^\T Vy$, and $\beta_w = ((X')^\T WX')^{-1}(X')^\T Wy'$.
    Assume $(\norm{v-w}_1+2)\leverage\le \frac 1 4$.
    We have $\norm{S_v^{1/2}(\beta_v-\beta_w)}^2 \le 4(\norm{v-w}_1 + 2)^2 \leverage\residual^2$.
\end{claim}

We use the following relationship between the closeness of covariance matrices and the indistinguishability of their induced Gaussians \citep[as in][]{brown2021covariance,alabi2022privately,duchi2023fast,brown2023fast}.
\begin{claim}\label{claim:covariance_indistinguishability}
    Fix $\eps\in (0,1)$ and $\delta\in (0,1/10]$ and let $\cov_1, \cov_2 \in \bbR^{d\times d}$ be positive definite matrices.
    If
        $\dpsd\del{\cov_1,\cov_2}
        \le \frac{\eps}{3\log 2/\delta}$
    then $\cN(0, \cov_1)\approx_{(\eps,\delta)} \cN(0, \cov_2)$.
\end{claim}

We also need two standard privacy facts: privacy of the Gaussian mechanism and the DP almost-triangle inequality.
\begin{fact}[Gaussian Mechanism]\label{claim:gaussian_mechanism_privacy}
    Fix $\eps,\delta\in(0,1)$ and let $u,v$ be vectors.
    If $\norm{u-v}_2\le \Delta$, then for any $c^2 \ge \Delta^2 \cdot \frac{2 \log 2/\delta} {\eps^2}$ we have  $\cN(u,c^2\bbI)\approx_{(\eps,\delta)}\cN(v,c^2\bbI)$.
\end{fact}

\begin{fact}[See \citet{vadhan2017complexity}]
\label{fact:dp_triangle}
    Suppose for some \(\varepsilon\) and \(\delta\) that distributions \(p_1\), \(p_2\), and \(p_3\) satisfy \(p_1 \approx_{\del{\varepsilon, \delta}} p_2\) and \(p_2 \approx_{\del{\varepsilon, \delta}} p_3\).
    Then \(p_1 \approx_{\del{2\varepsilon, \del{1 + e^\varepsilon}\delta}} p_3\).
\end{fact}

We are now ready to prove our main privacy claim.

\begin{lemma}[Main privacy guarantee]\label[lemma]{lemma:main_privacy}
    For \(\varepsilon \in \del{0, 1}\), \(\delta \in \intoc{0, \varepsilon/10}\), and \(\leverage_0, \residual_0 > 0\),
     \cref{alg:private_OLS} is \(\del{\varepsilon, \delta}\)-differentially private.
\end{lemma}
\begin{proof}
    Consider the execution of \cref{alg:private_OLS} on two adjacent datasets \(\del{X, y}\) and \(\del{X', y'}\), yielding \(\score_1, \score_2, v, \hat{\beta}\) and \(\score_1', \score_2', v', \hat{\beta}'\) respectively.
    Note that in order to not immediately fail, we must have 
    \[\leverage_0 \le \min\set*{\frac{1}{96k}, \frac{3\varepsilon}{56\log 12/\delta}}\]
    where \(k = \ceil*{\del{12 \log 3/\delta}/\varepsilon} + 8\).

    \paragraph{Privacy of the test} First, we will show that 
    \[\abs*{\max\set{\score_1, \score_2} - \max\set{\score_1', \score_2'}} \le 4.\]
    By \cref{thm:stable_cov_guarantees}, we have \(\abs{\score_1 - \score_1'} \le 2\).
    Without loss of generality, assume that \(\score_1 \ge \score_1'\).
    
    Considering the case where \(\score_1 = k\), we have \(\max\set{\score_1, \score_2} = k\) and
    \[\max\set{\score_1', \score_2'}\ge \score_1' \ge \score_1 - 2 \ge k - 2\]
    so in this case,
    \(\abs*{\max\set{\score_1, \score_2} - \max\set{\score_1', \score_2'}} \le 2\).
    
    Now if \(\score_1 < k\) then \(\score_1' < k\) as well, so we can apply \cref{claim:score_low_sensitivity} to get \(\abs{\score_2 - \score_2'} \le 4\).
    Then by noting that \(\max\) is \(1\)-Lipschitz in the \(\infty\)-norm, we have
    \begin{align}
        \MoveEqLeft{\abs*{\max\set{\score_1, \score_2} - \max\set{\score_1', \score_2'}}}\\
        &\le \max\set*{\abs{\score_1 - \score_1'}, \abs{\score_2 - \score_2'}}\\
        &\le \max\set{2, 4}\\
        &\le 4.
    \end{align}
    Finally we see that \(\PTRepsdelta{\varepsilon/3}{\delta/3}\del{\max\set{\score_1, \score_2}}\) is \(\del{\varepsilon/3, \delta/3}\)-DP by \cref{claim:ptr_guarantees}.
    
    \paragraph{Privacy of the parameter estimate} Now we will proceed under the assumption that 
    \[\PTRepsdelta{\varepsilon/3}{\delta/3}\del{\max\set{\score_1, \score_2}} = \PTRepsdelta{\varepsilon/3}{\delta/3}\del{\max\set{\score_1', \score_2'}} = \pass,\]
    with the goal of showing that \(\cN\del{\hat{\beta}, c^2S_v^{-1}} \approx_{2\varepsilon/3, 2\delta/3} \cN\del{\hat{\beta}', c^2S_{v'}^{-1}}\).
    Since the PTR checks passed, \cref{claim:ptr_guarantees} says that
    \[\score_1, \score_2, \score_1', \score_2' < k\]
    where \(k = \ceil*{\del{12 \log 3/\delta}/\varepsilon} + 8\), which matches the assignment in \cref{alg:private_OLS}.
    Now we can apply \cref{claim:weights_stable} to obtain \(\norm{v - v'}_1 \le 5\) and observe that \(v, v'\) are both \(\del{4\leverage_0, 2 \exp\del{216 k^2 \leverage_0} \residual_0}\)-good by \cref{claim:weights-good}. 
    We will use the stability and goodness of the weights to establish the stability of both \(\hat{\beta}\) and \(S_v\).
    
    \cref{claim:parameter_stability_from_weight_stability_different_datasets} requires \(28L_0 \le \frac{7}{24k} \le \frac{1}{4}\), which is true by assumption.
    The claim implies that  \(\norm{S_v^{1/2}\del{\hat{\beta} - \hat{\beta}}}^2 \le \Delta^2\) where \(\Delta^2 = 3136\exp\del{432 k^2 \leverage_0} \leverage_0\residual_0^2\).
    Next, we see that transforming \(\beta, \beta'\) by \(\del{S_v}^{-1/2}\) allows us to apply \cref{claim:gaussian_mechanism_privacy}, giving
    \[\cN\del{\hat{\beta}, c^2S_v^{-1}} \approx_{\varepsilon/3, \delta/6} \cN\del{\hat{\beta}', c^2S_{v}^{-1}}.\]
     as long as \(c^2 \ge \Delta^2 \cdot \frac{18\log 12/\delta}{\varepsilon^2}\), which is satisfied by construction in \cref{alg:private_OLS}.
    
    Then, since \(24\leverage_0 \le 1/\del{4k} \le 1/2\) by assumption, 
    \cref{claim:identifiability_covs} tells us that
        $\dpsd\del{S_v,S_{v'}}\le 56\leverage_0$.
    We apply \cref{claim:dpsd-inverse} and \cref{claim:covariance_indistinguishability} to obtain
    \[\cN\del{\hat{\beta}', c^2S_{v}^{-1}}\approx_{\varepsilon/3, \delta/6}\cN\del{\hat{\beta}', c^2S_{v'}^{-1}},\]
    since \(56\leverage_0 \le 3\varepsilon/\del{\log 12/\delta}\), which we assumed to be true.
    Finally, we apply \cref{fact:dp_triangle} to combine the two results, observing that \(e^\varepsilon < e < 3\), to complete the proof.
\end{proof}

\section{Utility Analysis of \cref{alg:private_OLS}}
\label{sec:utility} 

Given the privacy guarantee of \cref{lemma:main_privacy}, we analyze the utility of \cref{alg:private_OLS} under the standard subgaussian linear model. The definition of subgaussian variables and necessary concentration inequalities are provided in 
\cref{app:subgaussian_facts}. 
We first note that data from the standard subgaussian linear model is good with high probability.
\begin{lemma}[Subgaussian data is good]\label[lemma]{lemma:subgaussian_good}
    Let \(X \in \RR^{n \times d}\) be drawn i.i.d.\ from a \(d\)-dimensional subgaussian distribution \(\mathcal{D}\) with mean \(0\), (full-rank) covariance \(\Sigma\), and subgaussian parameter \(K_{\mathcal{D}}\).
    Let \(y_i = \beta^\T x_i + z_i\) where the \(z_i\) are drawn i.i.d. from a subgaussian distribution with mean 0, variance \(\sigma^2\), and subgaussian parameter \(K_\sigma\).
    There exists constants \(K_\leverage, K_\residual, K_n > 0\) such that for any \(\eta \in \del{0, 1}\), if \(n \ge K_nK_{\mathcal{D}}^4\del{d + \log\del{3/\eta}}\) then \(\del{X, y}\) is \(\del{\leverage, \residual}\)-good, where
    \[\leverage = K_\leverage K_{\mathcal{D}}^2 \cdot \frac{d + \log\del{3n/\eta}}{n}\qquad\text{and}\qquad \residual=K_\residual K_\sigma \sigma\sqrt{\log\del{3n/\eta}},\]
     with probability at least \(1 - \eta\).
\end{lemma}
\begin{proof}[Proof Sketch]
Identical calculations about leverage appeared in \cite{brown2021covariance,duchi2023fast,brown2023fast}.

Recall that we can write the vector of residuals $e=\hat y - y = (H-\bbI)z$, where $H$ is the hat matrix and $z$ the noise vector. 
Denote by $r_i$ the $i$-th row of $H-\bbI$.
Then $e_i = r_i^\T z$, which (for any fixed $H$) implies $e_i$ has subgaussian norm $\norm{r_i} K_{\sigma}$.
We know that $\norm{r_i}\le 1$, since $\bbI-H$ is idempotent and symmetric: $1\ge (\bbI-H)_{i,i} = r_i^\T r_i$.
Thus all $\abs{e_i}$ will be bounded with high probability.
\end{proof}

As we noted in \cref{obs:output_on_good_data}, when the input data is good \algname returns the OLS estimate plus noise.
We use a bound on the error of the OLS estimate from prior work.

\begin{lemma}[OLS error under random design, restatement of Theorem 1, \cite{hsu2011analysis}]\label[lemma]{lemma:ols_error}
Under the distributional assumption of Lemma~\ref{lemma:subgaussian_good}, there exists an absolute constant $K_{\mathrm{OLS}}$ such that, for any $\delta\in (0, 1)$, if $n>K_{\mathrm{OLS}}K_\mathcal{D}(d+\log(1/\delta))$, then with probability $1-\delta$, we have
\[
\norm[\big]{\hat{\beta}_{OLS}-\beta}^2_\Sigma \le \frac{K_{\mathrm{OLS}}K_\sigma^2\sigma^2(d+\log(1/\delta))}{n}.
\]
\end{lemma}

Now, we are ready to prove the main accuracy lemma by bounding the norm of the added noise.

\begin{lemma}[Main accuracy guarantee]\label[lemma]{lemma:accuracy_main}
    Let \(X \in \RR^{n \times d}\) be drawn i.i.d.\ from a \(d\)-dimensional subgaussian distribution \(\mathcal{D}\) with mean \(0\), (full-rank) covariance \(\Sigma\), and subgaussian parameter \(K_{\mathcal{D}}\).
    Let \(y_i = \beta^\T x_i + z_i\) where the \(z_i\) are drawn i.i.d. from a subgaussian distribution with mean 0, variance \(\sigma^2\), and subgaussian parameter \(K_\sigma\).
    There exists constants \(K_\leverage, K_\residual > 0\) such that for any \(\eta \in \del{0, 1}\), if
    \[\leverage_0 = K_\leverage K_{\mathcal{D}}^2 \cdot \frac{d + \log\del{3n/\eta}}{n},\qquad\qquad \residual_0=K_\residual K_\sigma \sigma \sqrt{\log\del{3n/\eta}},\]
    and
    \[n = \widetilde{\Omega}\del*{K_{\mathcal{D}}^4\del*{d+\log\del*{\frac{1}{\varepsilon\eta}}}\frac{ \del{\log 1/\delta}^2}{\varepsilon^2}},\]
    then with probability at least \(1 - \eta\) \cref{alg:private_OLS} successfully returns \(\tilde{\beta}\) such that
    \[\norm[\big]{\tilde{\beta} - \beta}_\Sigma \le \bigO\del*{K_\sigma\sigma\sqrt{\frac{d+\log\del{1/\eta}}{n}} +  K_{\mathcal{D}}K_\sigma \sigma \cdot \frac{\del{d + \log\del{n/\eta}}\sqrt{\log\del{n/\eta} \log\del{1/\delta}}}{\varepsilon n}},\]
    where \(\widetilde{\Omega}\) hides log factors in \(K_{\mathrm{D}}\) and \(\log1/\delta\).
\end{lemma}
\begin{proof}
    We begin by determining how many samples are needed to ensure that (i) the algorithm does not fail immediately and (ii) the data is $(\leverage_0,\residual_0)$-good (for the specified values) with high probability.
    
    In order to not fail, we require 
    \(\leverage_0 = \bigO\del*{\varepsilon/\log\del{1/\delta}}\).
    Meanwhile, in order to apply \cref{obs:output_on_good_data}, we require \(\leverage_0 = \bigO\del{\del{\varepsilon/\log\del{1/\delta}^2}}\).
    It is clear that the second requirement implies the first.
    Thus, we can expand our choice of \(L_0\) to get
    \[\leverage_0 = K_\leverage K_{\mathcal{D}}^2 \cdot \frac{d + \log\del{3n/\eta}}{n} = \bigO\del*{\frac{\varepsilon^2}{\del{\log 1/\delta}^2}}.\]
    Using the fact that \(a/\log a = \Omega\del{b}\) implies \(a = \Omega\del{b \log b}\), this translates to
    \[n = \Omega\del*{\frac{ K_{\mathcal{D}}^2 \del{\log 1/\delta}^2}{\varepsilon^2}\del*{d+\log\del*{\frac{K_{\mathcal{D}}^2 \del{\log 1/\delta}^2}{\varepsilon^2\eta}}}}.\]
    We note that this implies \(n = \Omega\del[\big]{K_\mathcal{D}\del{d + \log\del{1/\delta}}}\) as required by \cref{lemma:ols_error}.
    Now, in order to apply \cref{lemma:subgaussian_good}, we additionally require \(n = \Omega\del*{K_{\mathcal{D}}^4\del{d + \log\del{1/\eta}}}\). 
    As in \cref{lemma:subgaussian_good}, this requirement gives us \(\norm[\big]{\Sigma^{-1/2}\widehat{\Sigma}\Sigma^{-1/2} - \bbI}_2 \le 1/2\) by \cref{claim:concentration_of_covariance}, which we will use later.
    Combining the two outstanding requirements and dropping lower order terms gives
    \[n = \Omega\del*{K_{\mathcal{D}}^4\del*{d+\log\del*{\frac{K_{\mathcal{D}}^2 \del{\log 1/\delta}^2}{\varepsilon^2\eta}}}\frac{ \del{\log 1/\delta}^2}{\varepsilon^2}}.\]

    This ensures that \(\del{X, y}\) is \(\del{\leverage_0, \residual_0}\)-good with probability \(1 - \bigO\del{\eta}\).
    When this happens the PTR check passes deterministically.
    Thus, we now turn to evaluating the accuracy of our regression estimate.
    We apply the triangle inequality about $\ols$:
    \begin{align}
        \norm{\beta - \tilde\beta}_\Sigma 
            &\le \norm{\beta - \ols}_\Sigma + \norm{\ols - \tilde\beta}_\Sigma.
            \label{eq:accuracy_decomp}
    \end{align}
    We analyze these terms separately.

    The first term in \cref{eq:accuracy_decomp} is solely about the empirical quantity.
    By \cref{lemma:ols_error}, with probability at least \(1 - \bigO\del{\eta}\) we have
    \[\norm{\ols - \beta}_{\Sigma} = \bigO\del*{K_\sigma\sigma\sqrt{\frac{d+\log\del{1/\eta}}{n}}}.\]

    To bound the second term in \cref{eq:accuracy_decomp}, we apply \cref{obs:output_on_good_data}, which states that on good data \(\tilde{\beta}\) is drawn from 
        $\cN(\ols, c^2 \del{X^\T X}^{-1})$ where \(c^2=\Theta\del{\leverage_0 \,\residual_0^2\, {\log \del{1/\delta}}/{\eps^2}}\).
    Equivalently, we draw $z \sim \cN(0,\bbI)$ and set $\tilde \beta \gets \ols + c (X^\T X)^{-1/2} z$.
    Plugging this in, we have 
    \begin{align}
        \norm{\ols - \tilde\beta}_\Sigma &=  \norm{ c (X^\T X)^{-1/2} z}_\Sigma \\
        &=  c \cdot \norm{ \Sigma^{1/2} (X^\T X)^{-1/2} z}_2.
    \end{align}
    We plug in $\hat\Sigma = \frac 1 n X^\T X$ the empirical covariance and apply Cauchy--Schwarz:
    \begin{align}
        \norm{\ols - \tilde\beta}_\Sigma &\le \frac{c}{\sqrt{n}}\cdot\norm{\Sigma^{1/2}\hat\Sigma^{-1/2}}_2 \cdot \norm{z}_2. 
    \end{align}
    By \cref{claim:concentration_of_covariance} the matrix norm is at most a constant, and by \cref{claim:concentration_of_norm} we can bound $\norm{z}_2^2= O\del{d +\log 1/\eta}$ with probability at least $1-O(\eta)$.
    Plugging these in, along with our expressions for $c, \leverage_0,$ and $\residual_0$, we arrive at the expression in the lemma.
    Applying a union bound over the three failure cases finishes the proof.
\end{proof}

\section{Running-Time Analysis of \cref{alg:private_OLS}}\label{sec:running_time}

In this section we prove the following guarantee about the running time of \algname, whose computational requirements are quite lightweight.
The core ideas in this proof appeared in the analogous claim of BHS.

\begin{lemma}[Running Time]\label[lemma]{lemma:main_time}
    \cref{alg:private_OLS} can be implemented to require
    \begin{enumerate}
        \item one product of the form $A^\top A$ for $A\in \R^{n\times d}$,
        \item one product of the form $A B$ for $A\in R^{n\times d}$ and $B\in \R^{d\times d}$,
        \item one inversion of a positive definite matrix in $\R^{d\times d}$;
        and
        \item further computational overhead of $\tilde{O}(nd/\eps)$.
    \end{enumerate}
\end{lemma}
Ignoring bit complexity, this corresponds to time $\tilde{O}(nd^{\omega-1} + nd/\eps)$, where $\omega<2.38$ is the matrix multiplication exponent.
For modest privacy parameters, the running time of our algorithm is dominated by the time needed to compute the nonprivate OLS solution itself.

To establish this claim, we provide a second version of \stableOLS, \cref{alg:stable_OLS_efficient}, which is more computationally efficient.
We show that this alternative algorithm is functionally equivalent.

\begin{algorithm2e}
\SetAlgoLined
\SetKwInOut{Input}{input}
\SetKwInOut{Require}{require}
\Input{dataset \(X, y\); base outlier thresholds $\leverage_0,\residual_0$; weights $w$; discretization parameter $k$}
\BlankLine
$\forall j\in [2k], \residual_j \gets (\Rjump)^j \cdot \residual_0$\;
$\mathrm{COUNT} \gets 0$\;
\For{$j\in \{2k,2k-1,\ldots,0\}$}{
    \While{$\mathtt{TRUE}$}{
        \tcp{check for large residuals}
        $\beta_w \gets \weightedOLS(X,y,w)$\tcc*{via rank-one update}
        $i^* \gets \argmax_{i\in \supp{w}} \abs[\big]{y_i - x_i^\T \beta_w}$ \;
        \If{$\abs{y_{i^*}- x_{i^*}^\T \beta_w} \le \residual_j$ \emph{or} $\mathrm{COUNT} \ge k$}{
            \textbf{break}\tcc*{too many outliers or no large residuals}
        }
        $w_{i^*} \gets 0$\tcc*{otherwise, remove weight}
        $\mathrm{COUNT}\gets \mathrm{COUNT} + 1$\;
    
    }
    \If{$\mathrm{COUNT} \ge k$}{            
            \tcp{too many outliers}
            $\forall i \le j, \score^{(i)} \gets k$\;
            $\forall i \le j, u^{(i)} \gets 0^n$\;
            \textbf{break}\;
    }
    \tcp{store result and move to next threshold}
    $u^{(j)}\gets w$\;
    $\score^{(j)} \gets \min\{k, n - \lVert u^{(j)}\rVert_1 + j\}$\;
}
$\score\gets \min_{j\in \{0,\ldots,k\}} \score^{(j)}$\;
$v\gets \frac 1 k \sum_{j=k+1}^{2k} u^{(j)}$\;
\KwRet $\score, v$
\caption{\stableOLS, More Efficient Implementation}\label{alg:stable_OLS_efficient}
\end{algorithm2e}

\medskip 
\begin{proof}[Proof of \cref{lemma:main_time}]
    From BHS, Lemma 20 in Section 2.3, we see that we can implement \stablecovariance using one product $A^\top A$, one product $A B$, one matrix inversion, and at most $O(\log(1/\delta)/\eps)$ additional operations, each of which requires $\tilde{O}(nd)$ time.
    We need two additional conclusions from their analysis: \stablecovariance can be implemented to return the inverse weighted covariance $(X^\T W X)^{-1}$ in the same asymptotic running time and we can update all leverage scores in time $\tilde{O}(nd)$ when removing a single observation.

    With the weights $w$ and inverse covariance in hand, we call \stableOLS. 
    The initial regression parameter can be computed in $\tilde{O}(nd)$ time, as we compute the vector $X^\top W y$ with a matrix-vector product (since $W$ is diagonal) and multiply it with the inverse covariance.
    Computing all residuals is linear-time.
    
    Each outlier removal and associated set of updates can also be implemented in $\tilde{O}(nd)$ time. 
    This is because the removal of a single point corresponds to a rank-one update, which can be done efficiently.
    Recall from \cref{sec:tech} the equation for updating the least squares solution after removing a data point:
    \begin{align}
        \beta_{\mathrm{ols}(-j)} &= \ols + \frac{(X^\T X)^{-1} x_j}{1-h_j}\cdot \del[\big]{y_j - \ip{x_j}{\ols}}.
    \end{align}
    (A nearly identical formula applies when the data are weighted.)
    Since we have the previous leverage scores and inverse covariance, this update can be performed in time $O(nd)$.
    As before, with the new regression parameter all the residuals can be recalculated in linear time.

    Setting these details aside, we turn to the crux of the analysis: \stableOLS is functionally equivalent to \cref{alg:stable_OLS_efficient}, our efficient version.
    
    \cref{alg:stable_OLS_efficient} iterates through the residual thresholds in decreasing order.
    This is identical to independently calling the greedy algorithm repeatedly from scratch, since the removal process is deterministic (we can break ties in a consistent manner, e.g., using the index of the points). 
    Formally, for any $\residual > \residual'$ and any fixed $X,y,w$, the result of $\resthresh(X,y,\residual',w)$ is identical to calling $u\gets \resthresh(X,y,\residual,w)$ and then $\resthresh(X,y,\residual',u)$.

    \cref{alg:stable_OLS_efficient} also tracks a count of observations it removes and halts if that number reaches $k$.
    If it halts at level $\ell$, for all $j\le \ell$ it sets $u^{(j)}=0^n$ and $\score^{(j)}=k$.
    To show that this has no effect on the outcome of the algorithm, we suppose \cref{alg:stable_OLS_efficient}'s count reaches $k$ and analyze two cases.
    Let $\ell\in\{0,\ldots,2k\}$ be the residual threshold index at which the count $k$ was reached.
    We know that $u^{(\ell)}$ returned by $\resthresh(X,y,\residual_\ell,w)$ in \cref{alg:stable_OLS}, the main version of \stableOLS, satisfies $\lVert u^{(\ell)}\rVert_1 \le n-k$, since after any removal the weight is zero.
    This also holds for all $u^{(j)}$ with $j\le \ell$.

    \textbf{Case 1:} 
    Suppose $\ell>k$, i.e., $\ell$ falls among the indices used to compute the weights.
    Then for all $j\in\{0,\ldots,k\}$, the indices used to compute the scores, \cref{alg:stable_OLS} computes $u^{(j)}$ with $\lVert u^{(j)}\rVert_1 \le n-k$.
    This means \cref{alg:stable_OLS} computes $\score=k$, as does \cref{alg:stable_OLS_efficient} (since it sets $\score_j =k$ for all $j\le k$).
    (Recall that this causes \algname to fail deterministically, so the weights do not impact the output.)

    \textbf{Case 2:} 
    If $\ell\le k$, then $\ell$ falls among the indices used to compute the score.
    (Thus, \cref{alg:stable_OLS,alg:stable_OLS_efficient} return the same weights.)
    \cref{alg:stable_OLS_efficient} sets $\score^{(j)}=k$ for all $j\le \ell$.
    We claim that \cref{alg:stable_OLS} also computes $\score^{(j)}=k$ for all $j\le \ell$.
    To see this, recall that on these indices \cref{alg:stable_OLS} computes $u^{(j)}$ with $\norm{u^{(j)}}_1\le n-k$.

    To finish the proof, we note that the final $\beta_v$ and $S_v^{-1}$ computed by \algname can be computed with at most $k$ rank-one updates from their initial values.
    Since $k = O(\log(1/\delta)/\eps)$, we are done.
\end{proof}

\section*{Acknowledgements}
This work is supported in part by the National
Science Foundation under grant no. 2019844, 2112471, 2238080, and 2229876, NSF Graduate Research Fellowships Program, the Machine Learning Alliance at MIT CSAIL, and Microsoft Grant for Customer Experience Innovation.
Part of this work was done while GB was at Boston University. AS and GB (while at Boston University) were supported in part by NSF awards CCF-1763786 and CNS-2120667 as well as Faculty Awards from Google and Apple. JCP was supported in part by the Harvard Center for Research on Computation and Society.

\bibliography{bibliography}

\appendix

\section{Additional Related Work}
\label{sec:related} 

\paragraph{Private Mean Estimation}
Many of the developments in private linear regression have analogs in private mean estimation, albeit rearranged chronologically. 
Consider the canonical mean estimation problem with ``covariance-adaptive'' error guarantees, which respect the shape of the dataset: the error is measured in Mahalanobis norm with respect to the covariance matrix $\Sigma$ of the data $X$, $\|\Sigma^{-1/2}(\hat \mu-\mu) \|$. This scales each direction according to the directional variance, providing a more relevant measure of utility. This is closely related to how linear regression error corresponds to the $\Sigma$-norm, $\|\hat \beta-\beta\|_\Sigma$. 

For non-private mean estimation, the distinction between Euclidean and Mahalanobis norm error are minor. For example, the empirical mean achieves  such a geometry-aware guarantee and, like the OLS estimator, is accurate with roughly $d$ samples with no dependence on the condition number of the covariance of the data.

For private estimation, geometry-aware estimation is significantly more challenging, since the privately learning the geometry, i.e., the covariance matrix, is more sample expensive than the primary task of mean estimation. At the same time, it seemed like it was necessary to design the privacy noise that matches the shape of the data covariance. The standard Gaussian mechanism privately estimates the mean with only $d$ samples but its error depends polynomially on the condition number, a high price to pay when the estimator does not respect the geometry. 
The work of \cite{kamath_KLSU19} allows us to privately learn the covariance and apply the Gaussian mechanism on whitened data, but, as with the SSP approaches for linear regression, this requires $d^{3/2}$ samples. A long line of work that follows either makes more strict assumptions on the geometry or pays a price in the sample complexity \citep{KV17,biswas2020coinpress,cai2019cost,aden2021sample,bun2019private,bun2019average,liu2021robust,kamath2021private,hopkins2022efficient,alabi2022privately}. 
Of particular relevance for our work are frameworks introduced by \citet{tsfadia2022friendlycore} and \citet{ashtiani2021private}, which remove outliers in a way that depends on the rest of the dataset (e.g., asking that inliers be close to a large number of other examples). 
These frameworks bear some similarity to our techniques and to those of \cite{brown2023fast}, especially their ``Stable Mean'' estimator.
Informally, our approach improves over theirs in the ability to adapt the definition of \emph{outlier} and the resulting geometry as points are removed.

\cite{brown2021covariance} was the first to address this geometry-aware challenge in private mean estimation. 
They termed it the \emph{covariance estimation bottleneck} and gave two exponential-time approaches for avoiding it, achieving accurate estimation with $\tilde{O}(d)$ samples and no dependence on the condition number. 
The first, which combined the exponential mechanism with PTR (Propose-Test-Release), served as a direct inspiration to the HPTR (High-dimensional PTR) framework of \cite{liu2022differential}.
The concurrent works of \cite{duchi2023fast} and \cite{brown2023fast}  
built on the second algorithm of \cite{brown2021covariance}, giving time-efficient algorithms matching the guarantees of the exponential-time approaches. The sample complexity has linear dependence on the dimension $d$ and no dependence on the condition number $\kappa(X^\top X)$. 
As in \cite{duchi2023fast} and \cite{brown2023fast}, our goal is to achieve the same for linear regression.

\medskip
\noindent
\paragraph{Private Linear Regression.}
Commensurate with its centrality in statistical theory and practice,  significant effort has gone into producing differentially private algorithms for least squares \citep{vu2009differential,kifer2012private,mir2013differential,dimitrakakis2014robust,bassily2014private,wang2015privacy,foulds2016theory, minami2016differential}. 

One standard theme in many of these works is the class of assumptions that directly enable global sensitivity analysis.
Prime examples include assuming that the covariates satisfy an $\ell_2$ norm bound or that the true parameter lies in some ball about  the origin.
Such guarantees are incomparable with our definition of goodness (for example, our definition allows arbitrarily large covariates, but covariates with bounded norms may still have high leverage).
Under some collections of these assumptions, state-of-the-art guarantees are achieved in \citep{wang2018revisiting, sheffet2019old}, which in our setting translates into a sample complexity of $n= \Omega(d^{1.5}/(\alpha\varepsilon))$ to achieve $(1/\sigma) \|\hat\beta-\beta\|_\Sigma \leq \alpha$. 
Both these prior algorithms and ours analyze accuracy under the assumption that the input data is ``outlier-free." 
The prior work uses conditions on the norm of the covariates or the magnitude of the labels. These assumptions lend themselves handily to global sensitivity calculations. In contrast, our work uses a notion of outlier-freeness which is more in line with standard statistical practice: we ask that the dataset have no high-leverage or high-residual points. 

When applied to data from the standard sub-Gaussian linear models, \algname is the first computationally efficient algorithm to achieve linear dependence in the dimension $d$ and no dependence on the condition number $\kappa(X^\top X)$ (see \cref{thm:main}). This nearly matches the best known sample complexity of \cite{liu2022differential} that relies on an exponential time approach of HPTR: $n=\tilde O(d/\alpha^2 + (d+\log(1/\delta))/(\alpha \varepsilon)) $ samples suffice to achieve an error of $(1/\sigma)\|\hat\beta - \beta\|_\Sigma \leq \alpha$. 
Existing computationally efficient approaches based on gradient descent either assume the covariance matrix is close to identity \citep[][]{cai2019cost,brown2024private} or have polynomial dependence on the condition number $\kappa( X^\top X)$ \citep{varshney2022,liu2023label}. The best known sample complexity of an efficient algorithm is by \cite{liu2023label}: $n=\tilde O(d/\alpha^2 + (\kappa^{1/2} d\log(1/\delta))/(\alpha \varepsilon)) $.

\medskip\noindent
{\bf Iterative Thresholding.}
Our \resthresh algorithm is a special case of the family of iterative thresholding algorithms, a longstanding heuristic for robust linear regression that dates back to Legendre. 
Its theoretical properties in the non-asymptotic regime have been extensively studied recently in~\cite{bhatia2015robust, bhatia2017consistent, suggala2019adaptive, pensia2020robust, chen2022online}. \cite{shen2019learning} and \cite{NEURIPS2022_05b12f10} studied the iterative trimmed estimator under generalized linear models and \cite{shen2019iterative} studied the mixed linear regression setting. It is worth noting that most iterative thresholding algorithm in the robust linear regression setting will alternate between finding the OLS solution of the current set and finding the set with the smallest residual under the current regression coefficient, and no data point is permanently removed in each iteration. In contrary, our algorithm will permanently remove one data point in each iteration before recomputing the OLS solution.

\section{Preliminaries}\label{sec:preliminaries}

We collect here known preliminary results that we use in our analyses.

\begin{fact}\label{claim:dpsd-inverse}
    Let \(\cov_1, \cov_2\) be positive-definite matrices and define 
    \[\dpsd\del{\cov_1, \cov_2} =  \max\set*{\norm*{ \cov_1^{-1/2}\cov_2\cov_1^{-1/2}- \bbI }_{\tr}, \norm*{ \cov_2^{-1/2}\cov_1\cov_2^{-1/2}- \bbI }_{\tr}}.\]
    Then \(\dpsd\del{\cov_1, \cov_2} = \dpsd\del{\cov_1^{-1}, \cov_2^{-1}}\).
\end{fact}
\begin{proof}
    Note that \(\cov_1^{-1/2}\cov_2\cov_1^{-1/2}\) and \(\cov_2^{1/2}\cov_1^{-1}\cov_2^{1/2}\) are similar and likewise, \(\cov_2^{-1/2}\cov_1\cov_2^{-1/2}\) and \(\cov_1^{1/2}\cov_2^{-1}\cov_1^{1/2}\) are similar.
    Thus,
    \begin{align}
        \dpsd(\cov_1, \cov_2) 
        &= \max\set*{\norm*{ \cov_1^{-1/2}\cov_2\cov_1^{-1/2}- \bbI }_{\tr}, \norm*{ \cov_2^{-1/2}\cov_1\cov_2^{-1/2}- \bbI }_{\tr}}\\
        &= \max\set*{\norm*{ \cov_2^{1/2}\cov_1^{-1}\cov_2^{1/2}- \bbI }_{\tr}, \norm*{ \cov_1^{1/2}\cov_2^{-1}\cov_1^{1/2}- \bbI }_{\tr}}\\
        &= \dpsd\del{\cov_2^{-1}, \cov_1^{-1}}\\
        &= \dpsd\del{\cov_1^{-1}, \cov_2^{-1}}
    \end{align}
    as desired.
\end{proof}

\subsection{Subgaussian Random Variables and Concentration Inequalities}\label{app:subgaussian_facts}

For formal proofs of claims and further discussion we refer to~\cite{vershynin2018high}.

\begin{definition}[Subgaussian Norm]
    Let $\yy\in\bbR $ be a random variable. 
    The \emph{subgaussian norm of $\yy$}, denoted $\norm{\yy}_{\psi_2}$, is 
        $\norm{\yy}_{\psi_2} = \inf \braces{t>0 : \bbE \exp\paren{\yy^2/t^2}\le 2}$.
\end{definition}

\begin{definition}[Subgaussian Random Variable]\label[definition]{def:subgaussian_vector}
    Let $\yy\in\bbR^d$ be a random variable with mean $\mu$ and covariance $\Sigma$. 
    Call $\yy$ \emph{subgaussian with parameter $K$} if there exists $K\ge 1$ such that for all $v\in\bbR^d$ we have
    \begin{align}
        \norm{\ip{\yy-\mu}{v}}_{\psi_2} \;\; \le \;\; K \sqrt{ v^\T \Sigma v }.
    \end{align}
\end{definition}
For example, the Gaussian distribution $\cN(\mu,\Sigma)$ is subgaussian with parameter $K=O(1)$.

\begin{claim}[Concentration of Norm]\label{claim:concentration_of_norm}
    Let $\yy_1,\ldots,\yy_n$ be drawn i.i.d.\ from a $d$-dimensional subgaussian distribution with parameter $K_\yy>0$, mean $\mu$, and (full-rank) covariance $\Sigma$.
    There exists a constant $K_1>0$ such that, with probability at least $1-\beta$, we have both
    \begin{align}
        \norm*{\Sigma^{-1/2}\del{\yy_1 - \mu}}^2 \le K_1 K_{\yy}^2 \paren{d + \log 1/\beta}
  \quad\text{and}\quad 
        \norm*{\Sigma^{-1/2}\del*{\frac 1 n \sum_{i=1}^n \yy_i - \mu}}^2 \le K_1 K_{\yy}^2\cdot \frac{d + \log 1/\beta}{n}.
    \end{align}
\end{claim}

\begin{claim}[Concentration of Covariance]\label{claim:concentration_of_covariance}
    Let $\yy_1,\ldots, \yy_n$ be drawn i.i.d.\ from a $d$-dimensional subgaussian distribution with parameter $K_\yy>0$, mean $\mu=0$, and (full-rank) covariance $\Sigma$.
    Let $\hat\Sigma = \frac{1}{n}\sum_{i=1}^n \yy_i\yy_i^\T $ be the empirical covariance.
    There exist positive absolute constants $K_1$ and $K_2$ such that, for any $\beta\in (0,1)$, if $n\ge K_2\paren{d + \log 1/\beta}$, then
    with probability at least $1-\beta$ we have
    \begin{align}
        \norm{\Sigma^{-1/2}\hat\Sigma\Sigma^{-1/2} - \bbI}_2 \;\;\le\;\; K_1 K_{\yy}^2\sqrt{\frac{d + \log 1/\beta}{n}}.
    \end{align}
\end{claim}

\subsection{Details on \stablecovariance}
\label{sec:statement_of_stablecov}

As a preprocessing step, \algname performs a leverage-score filtering routine introduced by BHS.
The algorithm we use differs only superficially from their version.
(For instance, they compute a set of weights $w \in [0,1]^n$ and a weighted covariance estimate, while we only care about the weights themselves.)
For completeness, we now state the version we use here.
Recall that \cref{thm:stable_cov_guarantees} contains the relevant guarantees proved by BHS.
\begin{algorithm2e}[h]
\SetAlgoLined
\SetKwInOut{Input}{Input}
\Input{dataset \(X\in \bbR^{n\times d}\); outlier threshold \(\leverage_0\); discretization parameter $k\in \bbN$}
\BlankLine
$A \gets [n]$\;
\For{$j = 2k,2k-1,\ldots,0$}{
     $\leverage_j = \exp\{j /k\}\cdot \leverage_0$\;
     \Repeat{$\mathtt{OUT} = \emptyset$}{
        $S_A \gets \sum_{i\in A} x_i x_i^\T$\;
        $\mathtt{OUT} \gets \braces{i \in A : x_i^\T \paren{S_A}^{-1} x_i > \leverage_j}$\;
        $A \gets A \setminus \mathtt{OUT}$
    }
    $A_j \gets A$\;
}
$\score \gets \min \{k, \min_{0\le j \le k} \{n - \abs{A_j} + j\}\} $\;
\For{$i=1,\ldots,n$}{
    $w_i \gets \frac 1 k \sum_{j=k+1}^{2k} \indicator{i\in A_j}$\;
}
\KwRet $\score, w$\;

\caption{Stable Leverage Filtering (\stablecovariance), BHS}\label{alg:stablecovariance}
\end{algorithm2e}

\section{Deferred Proofs}
\label{sec:deferred_proofs}

We now give the proof for \cref{claim:final:goodness_after_removal}, which characterizes the effect of removing weighted points from a least squares model.
This is a natural generalization of standard results \citep{mendenhall2003second,belsley2005regression,huber2011robust}.
\begin{proof}[Proof of \cref{claim:final:goodness_after_removal}]
We start by setting up notation and bounding a term useful in proving both \eqref{eq:removal_leverage} and \eqref{eq:removal_prediction}.

    \paragraph{Setup}
    Assume without loss of generality that $\supp{w}= [n]$, as any points outside the support of $w$ are irrelevant. 
    Let $v = w' - w$, 
     $\norm{v}_1 = \rho$, 
    $W=\diag(w)$ (likewise $W'$ and $V$), and $C = X^\T W X$.
    Decompose $V = P N$ (for ``positive'' and ``negative'') where $P,N$ are diagonal matrices with $P_{i,i} = \sqrt{\abs{v_i}}$ and $N_{i,i} = \mathrm{sign}(v_i)\cdot \sqrt{\abs{v_{i}}}$.

    Let $\Delta = I + Y$ where $Y = NX C^{-1} X^\T P$. If $\|Y\|_{2} \leq \eps<1 $ then $I+ Y \succeq(1-\eps )I$. Consequently, $\Delta$ is invertible and $\|\Delta^{-1}\|_2 \leq 1 / (1-\eps)$. To prove that $\|Y\|_2 \leq \eps$, we use the fact that $\|Y\|_2 \leq \|Y\|_F$ and compute:
    \begin{align}
        \norm{NX C^{-1} X^\T P}_F^2 &= \sum_{i,j} \paren{ N_{i,i} P_{j,j} \cdot x_i^\T C^{-1} x_j  }^2 \\
        &\le \sum_{i,j} N_{i,i}^2 P_{j,j}^2 \leverage^2 \\
        &= \leverage^2 \sum_{i,j} \abs{v_i} \abs{v_j} \\
        &= \leverage^2 \sum_i \abs{v_i} \sum_j \abs{v_j} \\
        &= \leverage^2 \cdot \norm{w - w'}_1^2.
    \end{align}
    Ultimately, we arrive at $\norm{\Delta^{-1}}_2 \le (1- \norm{w-w'}_1 \leverage)^{-1}$. In the first line of the calculation above, we used the fact that,  by our initial assumption on the goodness of $(X,w)$,
\begin{align}
 \|x_i^\T C^{-1} x_j\|_2 \leq \norm{C^{-1/2} x_i}_2 \norm{C^{-1/2} x_j}_2  \leq L^2. 
\end{align}
Since, $\norm{w-w'}_1 \leverage \leq 1/2$, we conclude that $\|\Delta^{-1}\|_2\leq 2$.

    \paragraph{Bounding the Leverage Scores}
    Consider an index $i\in \supp{w}$ %
    Write out its leverage under $w'$ and plug in our notation:
    \begin{align}
        h_i' &= x_i^\T \paren{X^\T W' X}^{-1} x_i \\
            &= x_i^\T \paren{X^\T W X + X^\T V X}^{-1} x_i \\
            &= x_i^\T \paren{C + (PX)^\T NX}^{-1} x_i.
    \end{align}
    Recalling the fact that, $x_i^\T C^{-1} x_i = h_i$ and $\Delta = I + NXC^{-1}X^\T P$, we can apply the Woodbury matrix identity to arrive at the following relationship between $h_i'$ and $h_i$: 
    \begin{align}
        h_i' &= x_i^\T \left[C^{-1} - C^{-1} (PX)^\T \paren{{I + NX C^{-1} (PX)^\T }}^{-1} NX C^{-1}\right]x_i \\
        &= h_i - x_i^\T C^{-1} X^\T P \Delta^{-1} NX C^{-1}x_i. 
    \end{align}
    We want to upper bound this leverage, $h'_i$, so we take the absolute value of the right-hand term and apply Cauchy--Schwarz:
    \begin{align}
        \abs{x_i^\T C^{-1} X^\T P \Delta^{-1} NX C^{-1}x_i} 
            &\le \norm{\Delta^{-1}}_2 \cdot \norm{P X C^{-1} x_i}_2 \cdot \norm{NX C^{-1} x_i}_2.
    \end{align}
    We already argued that $\norm{\Delta^{-1}}_2\le 2$, so we turn to the second term in the product:
    \begin{align}
        \norm{PX C^{-1} x_i}_2^2 &= \sum_{j=1}^n P_{j,j}^2 (x_j^\T C^{-1} x_i)^2 \\
            &\le \sum_{j=1}^n \abs{v_i}\cdot \leverage^2 = \norm{w-w'}_1 \cdot \leverage^2.
    \end{align}
    Note that an identical bound also holds for $\norm{NXC^{-1} x_i}_2$, hence we have
    \begin{align}
        h_i' &\le h_i + 2 \norm{w-w'}_1 \leverage^2 \\
            &\le \leverage+ 2 \norm{w-w'}_1 \leverage^2 = \paren{1 + 2 \norm{w-w'}_1\leverage} \cdot \leverage.
    \end{align}

    \paragraph{Bounding the Residuals}
    As above, we use the Woodbury matrix identity to derive an expression for the regression line $\beta_{w'}$.
    \begin{align}
        \beta_{w'} &= \paren{X^\T W' X}^{-1} X^\T W' y \\
            &= \paren{X^\T W X + X^\T V X}^{-1} \paren{X^\T W y + X^\T V y} \\
            &= \paren{C + (PX)^\T NX}^{-1} \paren{X^\T W y + X^\T V y} \\
            &= \paren{C^{-1} - C^{-1} (PX)^\T \Delta^{-1} NX C^{-1}}\paren{X^\T W y + X^\T V y}.
    \end{align}
    This expression expands into four terms, which simplify nicely:
    \begin{align}
        \beta_{w'} &= \beta_w \\
            &\quad + C^{-1} X^\T V y \\
            &\quad - C^{-1} X^\T P \Delta^{-1} N X C^{-1} X^\T W y \\
            &\quad - C^{-1} X^\T P \Delta^{-1} N X C^{-1} X^\T V y \\
            &= \beta_w \\
            &\quad + C^{-1} X^\T P \Delta^{-1} N (P \Delta^{-1}N)^{-1} V y \\
            &\quad - C^{-1} X^\T P \Delta^{-1} N X \beta_w \\
            &\quad - C^{-1} X^\T P \Delta^{-1} N X C^{-1} X^\T V y \\
            &=\beta_w   + C^{-1} X^\T P \Delta^{-1} N \paren{N^{-1} \Delta P^{-1} V y - X \beta_w - X C^{-1} X^\T V y}.
    \end{align}
    We can further simplify the expression in the parentheses. In particular, by plugging in the definition of $\Delta$, we have that $N^{-1} \Delta P^{-1} V y - X \beta_w - X C^{-1} X^\T V y$ can be rewritten as:
    \begin{align}
            &N^{-1} \paren{I + NX C^{-1} X^\T P} P^{-1} V y - X \beta_w - X C^{-1} X^\T V y \\
            &= N^{-1} P^{-1}Vy + N^{-1} N X C^{-1} X^\T P P^{-1} Vy 
           - X \beta_w -  X C^{-1} X^\T Vy \\
            &= V^{-1} V y + XC^{-1} X^\T Vy  - X \beta_w - X C^{-1} X^\T Vy \\
            &= y - X\beta_w.
    \end{align}
    Plugging back in, we arrive at 
    \begin{align}
        \beta_{w'} &= \beta_w + C^{-1} X^\T P \Delta^{-1} N \paren{y - X\beta_w}.
            \label{eq:main_regression_change}
    \end{align}
To finish the proof, we consider the residual on a point $x_i\in \supp{w}$:
    \begin{align}
        \abs[\big]{y_i - x_i^\T \beta_{w'}} 
            &= \abs[\big]{y_i - x_i^\T \beta_w + x_i^\T \beta_w - x_i^\T \beta_{w'}} \\
            &\le \abs[\big]{y_i - x_i^\T \beta_w} + \abs[\big]{x_i^\T \beta_w - x_i^\T \beta_{w'}} \\
            &\le \residual + \abs[\big]{x_i^\T C^{-1} X^\T P \Delta^{-1} N \paren{y - X\beta_w}},
    \end{align}
    The bound $\abs[\big]{y_i - x_i^\T \beta_w}$ follows from our initial assumption that $w$ is $(L,R)$-good for $(X,y)$.
    We can bound the second term in an analogous way to how we bounded the leverage scores. In particular, using the fact that $\|N\|_2 \leq  \sqrt{\norm{w-w'}_1}$,
    \begin{align}
        \abs[\big]{x_i^\T C^{-1} X^\T P \Delta^{-1} N \paren{y - X\beta_w}}
            &\le \norm*{P X C^{-1} x_i}_2 \cdot \norm*{\Delta^{-1}}_2 \cdot \norm*{N (y - X \beta_w)}_2 \\
            &\le 2 \leverage \residual \norm{w - w'}_1.
    \end{align}
    This completes the proof. 
\end{proof}

We now prove \cref{claim:parameter_stability_from_weight_stability_different_datasets}, which says that, if two weight vectors on adjacent datasets are both good and close in total variation distance, then their least-squares solutions are close as well.
We start by considering the setting where the vectors correspond to the same dataset.
\begin{claim}\label{claim:parameter_stability_from_weight_stability}
    Assume $v,w$ are both $(\leverage,\residual)$-good for dataset $(X,y)$.
    Define $\cov_v = X^\T \mathrm{diag}(v) X$, $\beta_v = S_v^{-1}X^\T V y$, and likewise $\cov_w$, $\beta_w$.
    If $\norm{v-w}_1\leverage \le \frac 1 4$, then
        $\norm{\cov_v^{1/2}(\beta_v- \beta_w)}^2 \le 4 \norm{v -w }_1^2 \leverage \residual^2$.
\end{claim}

From this claim, \cref{claim:parameter_stability_from_weight_stability_different_datasets} is an easy corollary.
\begin{proof}[Proof of \cref{claim:parameter_stability_from_weight_stability_different_datasets}]
    Vectors $v,w\in[0,1]^n$ on adjacent datasets $(X,y)$ and $(X',y')$, respectively, correspond to vectors $v',w'\in [0,1]^{n+1}$ over the datasets' union. 
    We have $\norm{v'-w'}_1\le \norm{v-w}_1 + 2$.
\end{proof}
\begin{proof}[Proof of \cref{claim:parameter_stability_from_weight_stability}]
We start by expanding out the squared norm and substituting in the definition of $\cov_v$:
    \begin{align}
        \norm[\big]{\cov_v^{1/2}(\beta_v- \beta_w)}^2
            &= \angles[\big]{\beta_v -\beta_w, \cov_v (\beta_v- \beta_w)} \\
            &= \angles*{\beta_v -\beta_w, \paren{\sum_{i\in [n]} v_i \cdot x_i x_i^\T } (\beta_v- \beta_w)}.
    \end{align}
    Next we expand the sum across $\beta_v -\beta_w$ and add and subtract $\sum_i v_i \cdot x_i y_i$, which makes the right-hand side of the inner product look like a pair of gradients.
    \begin{align}
        \norm[\big]{\cov_v^{1/2}(\beta_v- \beta_w)}^2
            &= \angles*{\beta_v -\beta_w, \sum_{i}v_i \cdot x_i \angles{x_i, \beta_v} -\sum_{i}v_i \cdot x_i \angles{x_i, \beta_w}} \\
            &= \angles*{\beta_v -\beta_w, \sum_{i} v_i \cdot x_i \del[\big]{\angles{x_i, \beta_v} -y_i} - \sum_{i}v_i \cdot x_i \del[\big]{\angles{x_i, \beta_w} - y_i}}.
    \end{align}
    By definition, $\beta_v$ is the vector that sets the first gradient to zero, so we have
    \begin{align}
        \norm[\big]{\cov_v^{1/2}(\beta_v- \beta_w)}^2
            &= \angles*{\beta_v -\beta_w, 0 - \sum_{i}v_i \cdot x_i \del[\big]{\angles{x_i, \beta_w} - y_i}}.
    \end{align}
    We now add and subtract the gradient at $\beta_w$ weighted by $w$, which leaves a gradient term (also zero by definition) and the differences $v_i - w_i$:
    \begin{align}
        \norm[\big]{\cov_v^{1/2}(\beta_v- \beta_w)}^2
            &= \angles*{\beta_v -\beta_w, - \sum_{i}w_i \cdot x_i \del[\big]{\angles{x_i, \beta_w} - y_i} + \sum_{i}(w_i-v_i) \cdot x_i \del[\big]{\angles{x_i, \beta_w} - y_i}} \\
            &= \angles*{\beta_v -\beta_w, 0 - \sum_{i}(w_i-v_i) \cdot x_i \del[\big]{\angles{x_i, \beta_w} - y_i}} \\
            &= \sum_{i} \angles[\Big]{\beta_v -\beta_w, (w_i-v_i) \cdot x_i \del[\big]{\angles{x_i, \beta_w} - y_i}}.
    \end{align}
    We now insert $\cov_{v}^{1/2}\cov_v^{-1/2}$ in the middle of each inner product.
    We apply Cauchy--Schwarz to each term and pull out the scalars (recall that $w_i$ and $v_i$ are scalars, $x_i$ is a vector):
    \begin{align}
        \norm[\big]{\cov_v^{1/2}(\beta_v- \beta_w)}^2
            &= \sum_{i} \angles*{\cov_v^{1/2}(\beta_v -\beta_w), (w_i-v_i) \cdot \cov_v^{-1/2} x_i \del[\big]{\angles{x_i, \beta_w} - y_i}} \\
            &\le \sum_{i} \norm[\big]{\cov_v^{1/2}(\beta_v -\beta_w)} \cdot  \norm[\big]{(w_i-v_i) \cdot \cov_v^{-1/2}x_i \del[\big]{\angles{x_i, \beta_w} - y_i}} \\
            &= \sum_{i} \abs{v_i - w_i} \cdot \norm[\big]{\cov_v^{1/2}(\beta_v -\beta_w)} \cdot  \norm[\big]{\cov_v^{-1/2}x_i }\cdot \abs{\angles{x_i,\beta_w}-y_i} .
    \end{align}
    Both sides of the equation have a $\norm[\big]{\cov_v^{1/2}(\beta_v -\beta_w)}$ term; these cancel.
    We apply $\ell_1/\ell_\infty$ on the weight differences, leverage scores, and residuals.
    There is some subtlety here: define set $U = \supp{v}\cup\supp{w}$.
    Then we have
    \begin{align}
        \norm[\big]{\cov_v^{1/2}(\beta_v- \beta_w)}
            &\le \sum_{i} \abs{v_i - w_i} \cdot\norm{\cov_v^{-1/2} x_i}\cdot  \abs[\big]{\angles{x_i,\beta_w}-y_i}   \\
            &\le  \norm{v - w}_1 \cdot \paren{\max_{i\in U} \norm{\cov_v^{-1/2}x_i}\cdot \abs[\big]{\angles{x_i,\beta_w} - y_i}}.
    \end{align}
    By the definition of goodness, if $i\in \supp{v}$ then we have $\norm{\cov_v^{-1/2} x_i}\le \sqrt{\leverage}$.
    Similarly, if $i\in \supp{v}$, we have $\abs{y_i - x_i^\T \beta_w}\le \residual$.
    
    However, these bounds may not hold for points outside the relevant support.
    \cref{claim:identifiability_covs} allows us to bound the leverage: for all $i\in \supp{w}\cup\supp{v}$ we have $\norm{\cov_v^{-1/2}x_i}_2^2 \le 2\leverage$, since we assumed $(1+\norm{v-w}_1)\leverage\le \frac 1 2$.
    
    Similarly, bounding the residual involves a simple trick alongside \cref{claim:final:goodness_after_removal}.
    Let $\check{w}$ be the entry-wise minimum of $\{w,v\}$, so $\check{w}_i = \min\{w_i, v_i\}$.
    We have $\norm{\check{w} - w}_1, \norm{\check{w}-v}_1\le \norm{w-v}_1$ and, furthermore, the support of $\check{w}$ is contained in both the support of $w$ and that of $v$.
    Thus, we can apply \cref{claim:final:goodness_after_removal}: assuming $i\in \supp{w}$ (since otherwise $i\in \supp{v}$ and we have a bound on the residual)
    \begin{align}
        \abs{y_i - x_i^\T \beta_{v}} 
            &= \abs{y_i - x_i^\T \beta_w + x_i^\T \beta_w - x_i^\T \beta_{\check{w}} + x_i^\T \beta_{\check{w}} - x_i^\T \beta_{v}} \\
            &\le \abs{y_i - x_i^\T \beta_w} + \abs{x_i^\T \beta_w - x_i^\T \beta_{\check{w}}} + \abs{x_i^\T \beta_{\check{w}} - x_i^\T \beta_{v}} \\
            &\le \residual + 2 \norm{\check{w} - w}_1 \leverage \residual + 2 \norm{\check{w} - v}_1 \leverage \residual\\
            &\le \del[\big]{1 + 4 \norm{v - w}_1 \leverage} \residual.
    \end{align}
    Since $\norm{v-w}_1\leverage \le \frac 1 4$, this is at most $2\residual$.
\end{proof}

\section{Estimation of \(\sigma^2\)}
\label{app:sigma}
\begin{algorithm2e}[H]    
   \caption{Private $\sigma^2$ Estimator} 
   \label{alg:sigma_squared} 
   	\DontPrintSemicolon 
	\KwIn{$S=\{(x_i, y_i)\}_{i=1}^n$, target privacy $(\varepsilon_0,\delta_0)$, target failure probability $\zeta$.}
	\SetKwProg{Fn}{}{:}{}
	{ 
	Partition $S$ into $k=\lfloor C_1\log(1/(\delta_0\zeta))/\varepsilon\rfloor$ subsets of equal size and let $G_j$ be the $j$-th partition, where each dataset is of size $b=|G_j|=\lfloor n/k\rfloor$.\\

	For each $j\in [k]$, denote $\psi_j= \min_\beta(1/|G_j|) \sum_{i\in G_j}(y_i-\beta^\top x_i)^2$.\\	
	Partition $[0, \infty)$ into bins of geometrically increasing intervals $\Omega:= \left\{\ldots,\left[ 2^{-2/4}, 2^{-1 / 4}\right),\left[ 2^{-1 / 4}, 1\right),\left[1,2^{1 / 4}\right),\left[2^{1 / 4}, 2^{2/4}\right), \ldots\right\} \cup\{[0,0]\}$\\
	Run $(\varepsilon_0, \delta_0)$-DP histogram learner of Lemma~\ref{lem:hist-KV17} on $\{ \psi_j\}_{j=1}^k$ over $\Omega$ \\
	{\bf if} all the bins are empty {\bf then} 
	Return $\perp$\\
	Let $[\ell, r]$ be a non-empty bin that contains the maximum number of points in the DP histogram\\
	Return  $\ell$
	} 
\end{algorithm2e}

\begin{lemma}
\label{lem:sigma_squared}
	Algorithm~\ref{alg:sigma_squared} is $(\varepsilon_0, \delta_0)$-DP. Let $S=\{(x_i, y_i)\}_{i=1}^n$ be a dataset of i.i.d. samples with  $x_i \sim \cN(0, \Sigma)$, $y_i=x_i^\top \beta^*+z_i$ and $z_i\sim \cN(0, \sigma^2)$  %
for some unknown true parameter $\beta^*=\Sigma^{-1}\mathbb{E}[y_i x_i] \in \mathbb{R}^d$ and unknown $\Sigma$ and $\sigma^2$.  Suppose 
	\begin{align}
		n=\tilde{O}\left(\frac{d\log(1/(\delta_0\zeta))}{\varepsilon_0}\right)\;,
	\end{align}
	with a large enough constant then 
Algorithm~\ref{alg:sigma_squared} returns $\ell$ such that, with probability $1-\zeta$, 
\begin{align}
	\frac{1}{\sqrt{2}}\sigma^2\leq\ell \leq \sqrt{2}\sigma^2\;,
\end{align}
where $\tilde{O}(\cdot)$ hides logarithmic factors in $\log(1/\varepsilon_0),\log(1/\delta_0)$.
\end{lemma}
We provide a proof in Appendix.~\ref{app:proof_sigma_squared}.

\subsection{Proof of Lemma~\ref{lem:sigma_squared} on the private $\sigma^2$ estimation}
\label{app:proof_sigma_squared}
The privacy proof follows from the DP-histogram from Lemma~\ref{lem:hist-KV17}. We provide proof for utility.

For each of the partition $G_j$, we show $a_i:=\frac{1}{b}\sum_{i\in G_j}(y_i-\hat{\beta}_j^\top x_i)^2$ concentrates around the true parameter $\beta^*$ where $\hat{\beta}_j:=\argmin_\beta(1/|G_j|) \sum_{i\in G_j}(y_i-\beta^\top x_i)^2$. Let $f(\beta)=\frac{1}{b}\sum_{i\in G_j}(y_i-\beta^\top x_i)^2$. We know $f(\hat{\beta}_j) = \min_\beta \frac{1}{b}\sum_{i\in G_j}(y_i-\beta^\top x_i)^2 \leq  \frac{1}{b}\sum_{i\in G_j}(y_i-\beta^{*\top} x_i)^2=\frac{1}{b}\sum_{i\in G_j} z_i^2$

Since $z_i^2$ are sub-exponential, from Bernstein bound, we know  there exists constant $c_1>0$ such that with proability $1-\zeta$,
$\frac{1}{b}\sum_{i=1}^b z_i^2\leq \sigma^2(1+c\sqrt{\frac{\log(1/\zeta)}{b}}+c\frac{\log(1/\zeta)}{b})$.

Now we show lower bound of $f(\hat{\beta}_j)$. For any $\beta$, we also have
\begin{align}
	  f(\beta) = \frac{1}{b}\sum_{i\in G_j}(y_i-w^\top x_i)^2= \frac{1}{b}\sum_{i\in G_j}(z_i+x_i^\top(w^*-w))^2\;.
\end{align}

Let $\tilde \beta:=\left(\Sigma^{1 / 2}\left(\beta^*-\beta\right), \sigma\right) \in \mathbb{R}^{d+1}$ and $\tilde{x}_i:=\left(\Sigma^{-1 / 2} x_i, z_i / \sigma\right) \in \mathbb{R}^{d+1}$ for $i \in [n]$. By definition, we can see that $\tilde{x}_i$ is zero-mean sub-Gaussian with covariance $\mathbf{I}_{d+1}$.

\begin{align}
	f(\beta) = \frac{1}{b}\sum_{i\in G_j}(\tilde{\beta}^\top \tilde{x}_i)^2.
\end{align}

Following Lemma~9 from \cite{jambulapati2020robust}, we know for any vector $\tilde{\beta}$, there exists $c_2>0$ such that with probability $1-\zeta$, 

\begin{align}
	 \left|\frac{1}{b}\sum_{i\in G_j}(\tilde{\beta}^\top \tilde{x}_i)^2-\|\tilde{\beta}\|^2\right| =  \left|\tilde{\beta}^\top\left(\frac{1}{b}\sum_{i\in G_j} \tilde{x}_i\tilde{x}_i^\top-\mathbf{I}_d\right) \tilde{\beta}\right|
	 \leq c_2\sqrt{\frac{d+1+\log(1/\zeta)}{b}}+c_2\frac{d+1+\log(1/\zeta)}{b}
\end{align}

This means for any $w$, we have

\begin{align}
	f(\beta) = \frac{1}{b}\sum_{i\in G_j}^b(\tilde{\beta}^\top \tilde{x}_i)^2&\geq (1-c_2\sqrt{\frac{d+1+\log(1/\zeta)}{b}}-c_2\frac{d+1+\log(1/\zeta)}{b}) (\|\Sigma^{1/2}(\beta-\beta^*)\|^2+\sigma^2) \\
	&\geq (1-c_1\sqrt{\frac{d+1+\log(1/\zeta)}{b}}-c_2\frac{d+1+\log(1/\zeta)}{b}) \sigma^2\;.
\end{align}

Together with the upper bound, this implies that there exists constant $c_3>0$, such that with probability $1-\zeta$,

\begin{align}
	\left|\frac{1}{b}\sum_{i\in G_j} a_i-\sigma^2\right|=\left|f(\hat{\beta}_j)-\sigma^2\right|\leq c_3\left(\sqrt{\frac{d+1+\log(1/\zeta)}{b}}+\frac{d+1+\log(1/\zeta)}{b}\right)\sigma^2\;.
\end{align}

By union bound, there exists a constant $c_4>0$ such that if $b\geq c_4(d+\log(k/\zeta))$, then for all $j\in [k]$,

\begin{align}
	|\psi_j-\sigma^2| = |\frac{1}{b}\sum_{i\in G_j} a_i-\sigma^2|\leq 2^{1/8}\sigma^2\;.
\end{align} 

With probability $1-\zeta$, $\{\psi_j\}_{j=1}^k$  lie in interval of size $2^{1/4}\sigma^2$. Thus, at most two consecutive bins are filled with $\{\psi_j\}_{j=1}^k$. Denote them as $I=I_1\cup I_2$.

 Our analysis indicates that $\mathbb{P}(\psi_i\in I)\geq 0.99$. By private histogram in Lemma~\ref{lem:hist-KV17}, if $k\geq c_5\log(1/(\delta_0\zeta))/\varepsilon_0$, $|\hat{p}_I-\tilde{p}_I|\leq 0.01$ where $\hat{p}_I$ is the empirical count on $I$ and $\tilde{p}_I$ is the noisy count on $I$. Under this condition, one of these two intervals are released. This results in multiplicative error of $\sqrt{2}$.

\begin{lemma}[Stability-based histogram {\cite[Lemma~2.3]{karwa2017finite}}]\label{lem:hist-KV17} For every $K\in \mathbb{N}\cup \{\infty\}$, domain $\Omega$, for every collection of disjoint bins $B_1,\ldots, B_K$ defined on $\Omega$, $n\in \mathbb{N}$, $\eps\geq 0,\delta\in(0,1/n)$, $\beta>0$ and $\alpha\in (0,1)$ there exists an $(\eps,\delta)$-differentially private algorithm $M:\Omega^n\to \mathbb{R}^K$ such that for any set of data $X_1,\ldots,X_n\in \Omega^n$
\begin{enumerate}
\item $\hat{p}_k = \frac{1}{n}\sum_{X_i\in B_k}1$
\item $(\tilde{p}_1,\ldots,\tilde{p}_K)\gets M(X_1,\ldots,X_n),$ and
\item
$$
n\ge \min\left\{\frac{8}{\eps\beta}\log(2K/\alpha),\frac{8}{\eps\beta}\log(4/(\alpha\delta))\right\} 
$$
\end{enumerate}
then,
$$
\mathbb{P}(|\tilde{p}_k-\hat{p}_k|\le\beta)\ge 1-\alpha
$$
\end{lemma}

\section{Lower Bound}
\label{sec:lower_bound}

Series of advances have been made in designing tools for lower bounds in statistical estimation.

Fingerprinting \cite{narayanan2023better}. 

Our lower bound is a direct corollary of a similar lower bound on linear regression from 
\cite{cai2023score}.

 \begin{theorem}
 	Let $\cP_{\Sigma,  \sigma^2}$ be a class of distributions over $(x_i, z_i)\in \mathbb{R}^d \times \mathbb{R}$, where $x_i$ are i.i.d. samples from a \(d\)-dimensional subgaussian distribution with mean \(0\) 
        and covariance \(\Sigma \succ 0\), and \(z_i\) are i.i.d. samples from a subgaussian distribution with mean 0 and variance \(\sigma^2\)
        (see \cref{def:subgaussian_vector} in \cref{sec:preliminaries}). We observe labelled examples from linear model:  \(y_i = \beta^\T x_i + z_i\) with $\mathbb{E}[x_i z_i]=0$.
        Let $\cM_{\varepsilon,\delta}$ be a class of $(\varepsilon, \delta)$-DP estimators that are functions over the datasets $S=\{(x_i, y_i)\}_{i=1}^n$. Then if $0<\varepsilon<1$, $d\lesssim n\varepsilon$, $\delta \lesssim n^{-(1+\gamma)}$ for some $\gamma>0$, there exists constant $C>0$ such that 
        \begin{align}
        	\inf _{M \in \mathcal{M}_{\varepsilon, \delta}} \sup_{\cP_{\Sigma,  \sigma^2}, \beta} \mathbb{E}\|M(y, x)-\beta\|_\Sigma^2 \geq C \sigma^2\left(\frac{d}{n}+\frac{d^2}{n^2 \varepsilon^2}\right).
        \end{align}
 \end{theorem}

 \begin{proof}
 	We will apply the lower bound below from {\cite{cai2023score}}.
 	\begin{theorem}[{\cite[Theorem 3.1]{cai2023score}}]
 	\label{thm:lower_cai}
Consider i.i.d. observations $\{\left(y_1, x_1\right), \cdots,\left(y_n, x_n\right)\}$ drawn from the Gaussian linear model:
\begin{align}
 	f_{\beta}(y \mid x)=\frac{1}{\sqrt{2 \pi} \sigma} \exp \left(\frac{-\left(y-x^{\top} \beta\right)^2}{2 \sigma^2}\right) ; x \sim f_{x}\;.
 \end{align}
Suppose $\mathbb{E}[xx^\top]$ is diagonal, and $\lambda_{\max}(\mathbb{E}[xx^\top])< C'<\infty$,  $\|X\|_2 \lesssim \sqrt{d}$ almost surely. If $d  \lesssim n \varepsilon, 0<\varepsilon<1$ and $\delta \lesssim n^{-(1+\gamma)}$ for some $\gamma>0$, then
\begin{align}
	\inf _{M \in \mathcal{M}_{\varepsilon, \delta}} \sup _{\beta \in \mathbb{R}^d} \mathbb{E}\|M(y, x)-\beta\|_2^2 \gtrsim \sigma^2\left(\frac{d}{n}+\frac{d^2}{n^2 \varepsilon^2}\right)\;.
\end{align}
 \end{theorem}
Note that this lower bound holds for every construction of $x_i$ that satisfies the assumption. We construct one instance of joint distribution $P\in \cP_{\Sigma, \sigma^2}$ such that it also satisfies the assumptions in Theorem~\ref{thm:lower_cai}. Let $\{x_i\}_{i=1}^n$ be i.i.d. samples from $\cN(0, \mathbf{I}_d)$. And let $\tilde{x}_i=x_i\cdot \mathbf{I}[x_i\leq \sqrt{d}]$.  Clearly, $\{\tilde{x}_i\}_{i=1}^n$ satisfies that $\mathbb{E}[\tilde{x}\tilde{x}^\top]$ is diagonal, $\lambda_{\max}(\mathbb{E}[\tilde{x}\tilde{x}^\top])< 1$ and $\tilde{x}_i$ are bounded by $\sqrt{d}$. Let $z_i$ be independent Gaussian distribution with variance $\sigma^2$. By Theorem~\ref{thm:lower_cai}, we know there exists constant $C$ such that
\begin{align}
	\inf _{M \in \mathcal{M}_{\varepsilon, \delta}} \sup_{\cP_{\Sigma,  \sigma^2}, \beta} \mathbb{E}\|M(y, x)-\beta\|_\Sigma^2 \geq C \sigma^2\left(\frac{d}{n}+\frac{d^2}{n^2 \varepsilon^2}\right).
\end{align}
 	
 \end{proof}

\end{document}